\documentclass[12pt]{article}
\usepackage{preprint-layout} 
\usepackage{preprint-notation} 
\usepackage{hyperref}
\usepackage[english]{babel}
\usepackage{doi}

\title{The Geometric Structure of Fully-Connected ReLU Layers}

\author{Jonatan Vallin, \quad Karl Larsson, \quad Mats G. Larson}

\date{\today}
\begin{document} 

\maketitle

\begin{abstract}
We formalize and interpret the geometric structure of $d$-dimensional fully connected ReLU layers in neural networks. The parameters of a ReLU layer induce a natural partition of the input domain, such that the ReLU layer can be significantly simplified in each sector of the partition. This leads to a geometric interpretation of a ReLU layer as a projection onto a polyhedral cone followed by an affine transformation, in line with the description in \cite[\doi{10.48550/arXiv.1905.08922}]{carlsson2019geometry} for convolutional networks with ReLU activations. Further, this structure facilitates simplified expressions for preimages of the intersection between partition sectors and hyperplanes, which is useful when describing decision boundaries in a classification setting. We investigate this in detail for a feed-forward network with one hidden ReLU-layer, where we provide results on the geometric complexity of the decision boundary generated by such networks, as well as proving that modulo an affine transformation, such a network can only generate $d$ different decision boundaries. Finally, the effect of adding more layers to the network is discussed.
\end{abstract}



\section{Introduction}

The most popular non-linear activation functions in deep learning today are variants of the Rectified Linear Unit (ReLU) -- in its standard form defined $\mathrm{ReLU}(x) = \max(0,x)$.
This popularity is due to its state-of-the-art performance, both regarding efficiency, thanks to computational simplicity, and training, thanks to mostly linear behavior and alleviation of issues such as vanishing gradients \cite{deepsparse2011, DUBEY202292}.
The focus of this paper are fully-connected layers with ReLU activations, herein denoted \emph{ReLU layers}, on the from
\begin{align} \label{eq:fundmap}
\boxed{T(x) = \mathrm{ReLU}(Ax + b)}
\end{align}
where the $\mathrm{ReLU}$ is applied component-wise and the matrix $A$ and vector $b$
are the layer's training parameters.
These are fundamental building blocks of common deep network architectures, for instance, convolutional and feed-forward networks, and an increased theoretical understanding of their behavior is key for fully understanding the properties of the architectures that include them.

\paragraph{Contributions.}
In \cite{carlsson2019geometry} Carlsson provides a geometric description for the action of a layer on the form \eqref{eq:fundmap}, and uses it to give a procedure for computing preimages of deep convolutional networks. We here further formalize and interpret this geometric description, which we then use to show properties of decision boundaries generated by feed-forward networks. Our main contributions are summarized in the following points.
\begin{itemize}
\item We formalize the geometric description of a ReLU layer in \cite{carlsson2019geometry} by introducing a detailed set notation for the natural partitions of the layer domain and codomain induced by the parameters.
This facilitates explicit expressions for the images and preimages of the layer and a geometric interpretation of a ReLU layer as a projection onto a polyhedral cone followed by an affine transformation.
The description is generalized to include contracting ReLU layers where the dimension is reduced. In an upcoming paper, this geometric description will serve as a basis for deriving error bounds for approximating hypersurfaces by decision boundaries of deep ReLU networks.
\end{itemize}
We utilize the formalized description in a binary classification problem, where the decision boundary separating two classes is defined as the zero-contour of a $d$-dimensional fully-connected feed-forward network with one hidden ReLU layer $T:\IR^d\to\IR^d$ and a final affine transformation $L:\IR^d\to\IR$. With the exception of very specific parameter configurations, we prove that:
\begin{itemize}
\item The number of linear pieces of the decision boundary is precisely $2^d-2^m$, where $m$ is an integer given by the parameters in the ReLU layer via the geometric description.
\item Modulo an affine transformation, such a network can only generate $d$ different decision boundaries.
\end{itemize}
Further, we discuss how the class of decision boundaries is affected by adding additional hidden ReLU layers to the network.

\paragraph{Previous Works.}
While there is an abundance of empirical studies of various properties of networks with ReLU layers, we here mainly focus on theoretical results for finite ReLU networks on the form
\begin{align}
F(x) = L \circ T^{(N)} \circ \dots \circ T^{(1)}(x)
\end{align}
The basis of many studies, including the present work, is the fundamental observation that ReLU networks are continuous piecewise linear functions. A measure of geometric complexity is the number of linear pieces such networks produce, which can grow exponentially in the number of layers \cite{montufar2014number}. These pieces are however highly dependent of each other, and in practice, deep networks may only use a portion of their theoretical maximum expressiveness \cite{hanin2019deep}. By specific parameter choices ReLU networks can be constructed to represent the maximum operation, from which it can be deduced that any continuous piecewise linear function can be represented by a ReLU network with sufficiently many parameters \cite{pmlr-v28-goodfellow13, hanin2017approximating, pmlr-v80-balestriero18b}, which in turn implies that such networks can approximate smooth functions. 

From a geometric point of view, a deep neural network can be seen as a sequence of mappings that gradually transforms seemingly geometrically complex input data to something manageable \cite{10.1007/978-3-642-24412-4_3}, where each layer in general simplifies the data's shape \cite{basri2017efficient,10.1007/978-3-540-45080-1_66}, or even it's topology \cite{10.5555/3455716.3455900}. A geometric description of how a ReLU layer on the form \eqref{eq:fundmap} transforms the data, as a projection onto a polyhedral cone followed by an affine transformation, is given in \cite{carlsson2019geometry}. This description is based on a dual basis induced by a geometric interpretation of the layer's parameters \cite{CarlssonARSS17}, and is utilized for computing preimages of convolutional networks.

As the decision boundaries of deep neural networks characterize the learned classifier, it is essential to understand their mathematical properties including their geometry and complexity. For instance, the works \cite{fawzi2018empirical, moosavi2019robustness} present a connection between geometrical properties of the decision boundaries and the robustness of the classifier. The authors provide evidence that there is a strong relation between the sensitivity of perturbation of the input data and large curvature of the decision boundary of the network. Along this line, the authors of \cite{liu2022some} use tools from differential geometry to derive sufficient conditions on the network parameters for producing flat or developable decision boundaries. They also provide a method to compute topological properties of the decision boundary.

Another geometric viewpoint is presented in  \cite{pmlr-v80-zhang18i}, where ReLU networks are described in terms of tropical geometry. This provides a connection between properties of the network and tropical geometric objects, where for instance the ReLU layer \eqref{eq:fundmap} is characterized by the tropical zonotopes. This work was recently extended in \cite{alfarra2022decision} where they provide a geometrical description of the decision boundary for a shallow network model. They prove that the decision boundary is contained in the convex hull of two zonotopes derived from the network parameters.

\paragraph{Outline.} In Section~\ref{sec:2}, we describe the structure of the mapping defined by a standard ReLU layer. Inspired by \cite{carlsson2019geometry}, we introduce a convenient dual basis obtained through the parameters in the layer, and using this dual basis, we construct a partition of the input space, allowing us to describe the action of the mapping explicitly.
In Section~\ref{sec:3}, we utilize this description in a classification setting. We characterize in detail the geometry of decision boundaries generated by a shallow ReLU network and provide a high-level description of the effect on decision boundaries when adding more layers to a network.
In Section~\ref{sec:conclusion}, we summarize our findings.

\section{The Geometrical Structure of a ReLU Layer}
\label{sec:2}
In this section, we analyze the geometrical structure of fully-connected ReLU layers and derive expressions for their preimages.
The map of a ReLU layer \eqref{eq:fundmap} can be written
\begin{align}
\boxed{
T:\IR^d \ni x \mapsto \mathrm{ReLU}(A_b(x)) \in \IR^d_+
}
\label{eq:map}
\end{align}
where $\IR^d_+$ denotes the non-negative orthant in $\IR^d$ and $A_b:\IR^d\to\IR^d$ is the affine map
\begin{equation}
\boxed{
A_b(x)=Ax+b
}
\label{eq:A_b}
\end{equation}
with parameters $A\in \IR^{d\times d}$ and $b\in \IR^d$.

\paragraph{Geometry of the Affine Map.}
We begin by giving a geometric interpretation of the parameters in the affine map \eqref{eq:A_b}.
Let row $i$ in the affine map be denoted $\rho_i(x)$, i.e.,
\begin{align}
\rho_i(x) = a_i \cdot x + b_i  = (A x + b)_i \quad \text{for } i=1,\hdots,d
\end{align}
where $a_i\in \IR^d$ is the $i$:th row of the matrix $A$ and $b_i\in \IR$ is the $i$:th element of the vector $b$. The zero levels for $\{\rho_i\}_{i=1}^d$ define hyperplanes 
\begin{align}
P_{i}= \{x\in \IR^d : \rho_i(x) = 0 \}  \quad \text{for } i=1,\hdots,d
\label{eq:hyppp}
\end{align}
with normals $a_i$, and we let the sign of $\rho_i$ define half-spaces
\begin{align}
U_{i,+} = \{x\in \IR^d : \rho_i(x) > 0 \},
\qquad 
U_{i,-} = \{x\in \IR^d : \rho_i(x) < 0 \}
\end{align}
Note that $\rho_i$ is a scaled version of the signed distance function associated with $P_i$, which is positive 
on $U_{i,+}$ (the half-space into which $a_i$ is directed) and negative on $U_{i,-}$. Further, assume that $\{a_i\}_{i=1}^d$ spans $\IR^d$ so that the hyperplanes are in general position. Then, the intersection of the hyperplanes is a point $x_0 = \cap_{i=1}^d P_i$, which is the unique solution to the linear system of equations 
\begin{align}
A x_0 + b = 0
\end{align} 
Let $I = \{ 1, \dots, d\}$ and define $L_i$ to be the line 
\begin{align}
L_i = \bigcap_{j \in I \setminus \{i\} } P_j
\end{align}
passing through $x_0$. Since the hyperplanes are assumed to be in general position, the line $L_i$ and the hyperplane $P_i$ will only coincide at $x_0$. Nevertheless, for $j\neq i$ we have that $L_i\subset P_j$ by definition. In accordance with the work in \cite{carlsson2019geometry}, for $i\in I$ we will let $a_i^*\in \IR^d$ be a vector parallel to $L_i$ directed such that $x_0 + a_i^*\in U_{i,+}$. Hence, $a_j \cdot a^*_i=0$ for $j\neq i$ and $a_i\cdot a^*_i>0$, and by assigning a length to each vector $a_i^*$ these vectors become uniquely determined, which we summarize in the following definition.

\begin{definition}[Dual Basis] \label{def:dual-basis}
Given an invertible matrix $A\in \IR^{d\times d}$ with rows $a_i\in \IR^d$, we define the set of vectors $\{a_i^*:i\in I\}$ satisfying
\begin{align}
\label{eq:deltaij}
\boxed{
a_j \cdot a_i^*=\delta_{ij} \quad \text{for} \quad i,j\in I
}
\end{align}
and denote this set the dual basis of $A$.
\end{definition}
Since the vectors $\{a_i : i\in I\}$ are assumed to be linearly independent, the vectors in the dual basis $\{a_i^* : i \in I\}$ will also be linearly independent, and hence the dual basis is also a basis in $\IR^d$. The dual basis will be useful for describing the action of a ReLU layer \eqref{eq:map}. Figure~\ref{fig:dualbasis} depicts the geometrical construction of the dual basis. Algebraically, the vector $a_i^*$ is the $i$:th column vector of the inverse matrix $A^{-1}$ and, hence, the action of the matrix $A$ on the vector $a_i^*$ is simply
\begin{align}
A a_i^* = e_i
\label{eq:dual}
\end{align}
where $e_i$ is the $i$:th basis vector in the standard Euclidean basis.
By expanding $x\in \IR^d$ in the dual basis, such that
\begin{align}
x = x_0 + \sum_{i\in I} \lambda_i a_i^*
\label{eq:dualexpansion}
\end{align}
with coefficients $\lambda_i \in \IR$, and applying the affine map \eqref{eq:A_b} we have
\begin{align}
A_b(x) =
A x + b = \underbrace{A x_0 + b}_{=0} +    \sum_{i\in I} \lambda_i (A a_i^*)
=
\sum_{i\in I} \lambda_i e_i=[\lambda_1,\lambda_2,\hdots,\lambda_d]^T
\label{eq:affine}
\end{align}
Thus, applying the affine map yields a vector with the coefficients of the expansion in the dual basis $\{a_i^*:i\in I\}$ as its elements.

\begin{figure}
\centering
\includegraphics[width=0.5\linewidth]{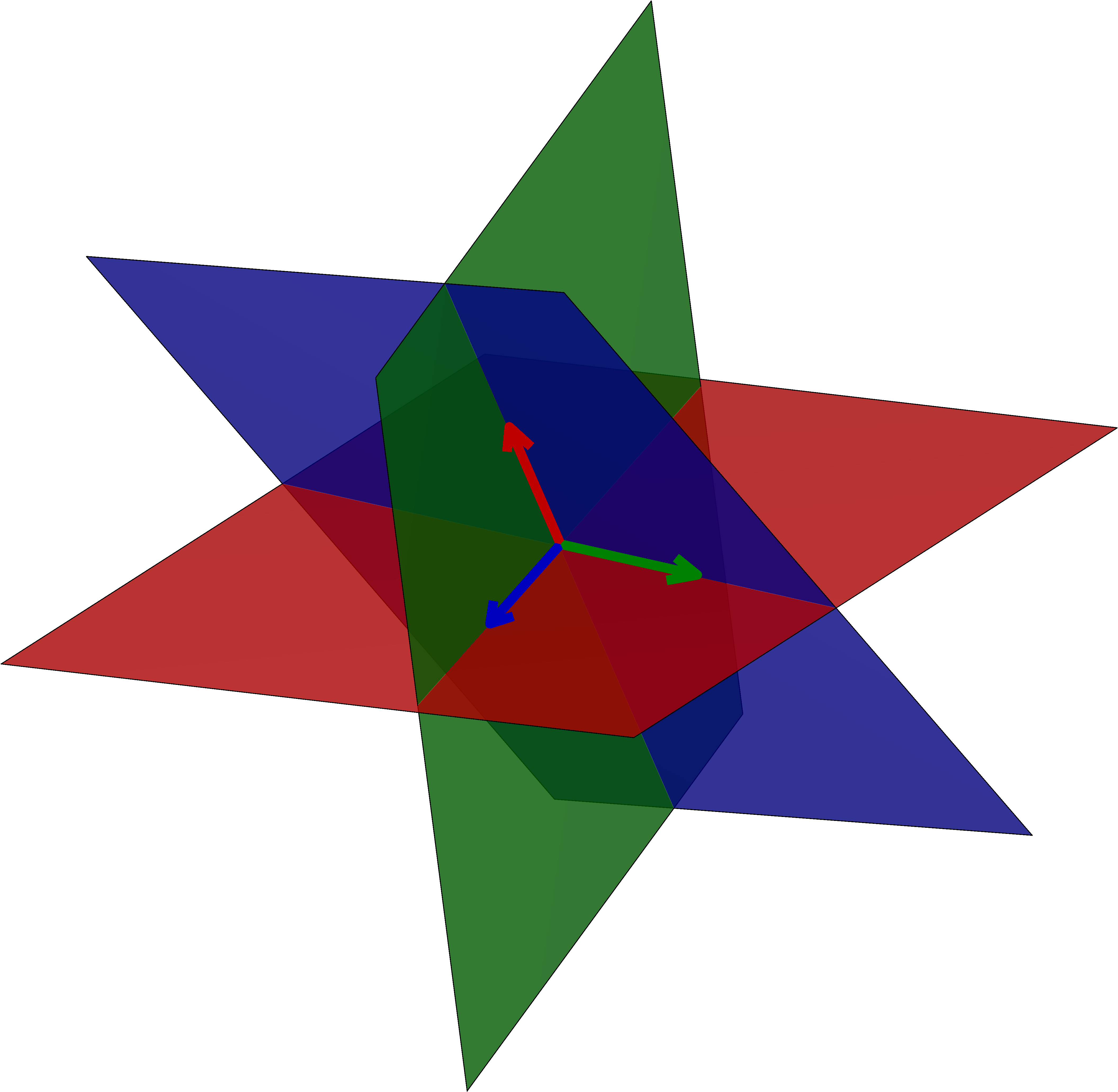}
\caption{\emph{Dual Basis.} The geometrical construction of the dual basis given a set of hyperplanes in $\IR^3$. The dual vectors are color-coded such that $a^*_i$ has the same color as the hyperplane with normal $a_i$. A dual vector $a^*_i$ is parallel to the line given by the intersection of the set of hyperplanes $\{P_j:j\in I\setminus \{i\}\}$ where $P_i$ is removed.}
\label{fig:dualbasis}
\end{figure}

\subsection{Partition of the Domain and Codomain}
In this section, we will introduce a partition of the ReLU layer domain --- the $\IR^d$ input, and a partition of the ReLU layer codomain -- the $\IR^d$ output.
These partitions will be useful in describing the action of the layer.
Recall that $I = \{ 1, \dots, d\}$ and consider two disjoint index subsets $I_+,I_- \subseteq I$ where $I_+\cap I_- = \emptyset$, denoted by the pairing $\bfI = (I_+, I_-)$.
Let $\mcI$ be the set of all such pairings. 

\paragraph{Domain Partition.}
For a given $(I_+,I_-)\in \mcI$ we define the following subset of $\IR^d$
\begin{align}
S_{(I_+,I_-)} = \bigg\{ x\in \IR^d:x=x_0+\sum_{i \in I_+} \alpha_i a_i^* - \sum_{i \in I_-} \alpha_i a_i^*, \text{ with } \alpha_i>0\bigg\}
\label{eq:S}
\end{align}
The family of all such sets
\begin{align}
\mcS = \{ S_{\bfI} \subset \IR^d \, : \, \bfI \in \mcI \}
\end{align}
will serve useful in describing the action of the ReLU layer in different parts of the domain.
We first verify that $\mcS$ constitutes a partition of $\IR^d$.
Consider a general point $x\in \IR^d$ expanded in the dual basis \eqref{eq:dualexpansion}, and define the pairing
\begin{align}
(I_+,I_-) =
\bigl(\{i\in I:\lambda_i>0\}, \{i\in I:\lambda_i<0\}\bigr)
\end{align}
The point $x\in \IR^d$ can then be expanded on the form
\begin{align}
x
=x_0 + \sum_{i \in I} \lambda_i a_i^*
=x_0+\sum_{i \in I_+} \alpha_i a_i^* - \sum_{i \in I_-} \alpha_i a_i^*
\label{eq:expansion}
\end{align}
with coefficients $\alpha_i=|\lambda_i|>0$ for $i\in I_+ \cup I_-$,
and we recognize this as the structure of points in \eqref{eq:S}.
Hence, each point in $\IR^d$ belongs to precisely one $S_{\bfI}$, which in turn means that $\mcS$ defines a partition of $\IR^d$, such that
\begin{align}
\IR^d = \bigcup_{\bfI \in \mcI} S_{\bfI}
\qquad\text{and}\qquad
\emptyset = S_{\bfI}\cap S_{\bfJ} \quad \text{for} \quad \bfI\neq \bfJ
\label{eq:partitionS}
\end{align}
When we refer to the dimension of a set $S_{(I_+,I_-)}\in \mcS$ we mean the dimension of the subspace span$(\{a_i^*:i\in I_+\cup I_-\})$, and since the dual vectors $a_i^*$ are assumed to be linearly independent we simply get $\dim(S_{(I_+,I_-)})=|I_+\cup I_-|$. By a simple combinatorial argument, it is easy to verify that the number of $k$ dimensional sets in the partition $\mcS$ is precisely equal to $\binom{d}{k}2^k$, so the total number of sets in $\mcS$ is
\begin{align}
|\mcS|=\sum_{k=0}^d\binom{d}{k}2^k=3^d
\end{align}
by the binomial theorem. An example of a partition $\mcS$ generated by a dual basis in $\IR^2$ is illustrated in Figure \ref{fig:dualbasispartition}.
\begin{figure} 
\centering
\begin{subfigure}[b]{0.35\textwidth}
\centering
\includegraphics[width=\textwidth]{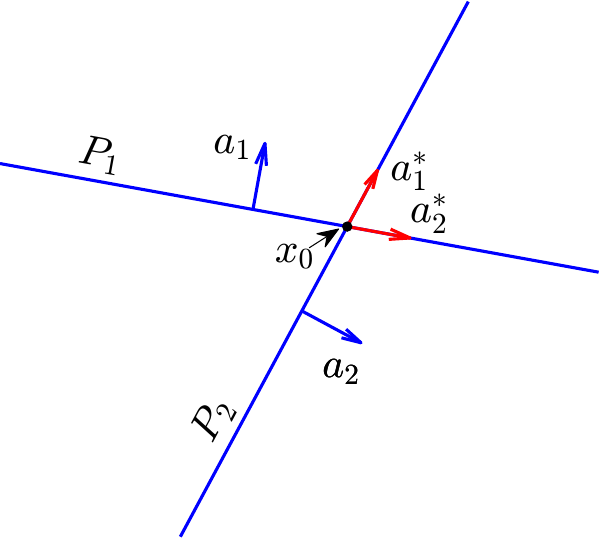}
\caption{ }
\end{subfigure}
\quad
\begin{subfigure}[b]{0.4\textwidth}
\centering
\includegraphics[width=\textwidth]{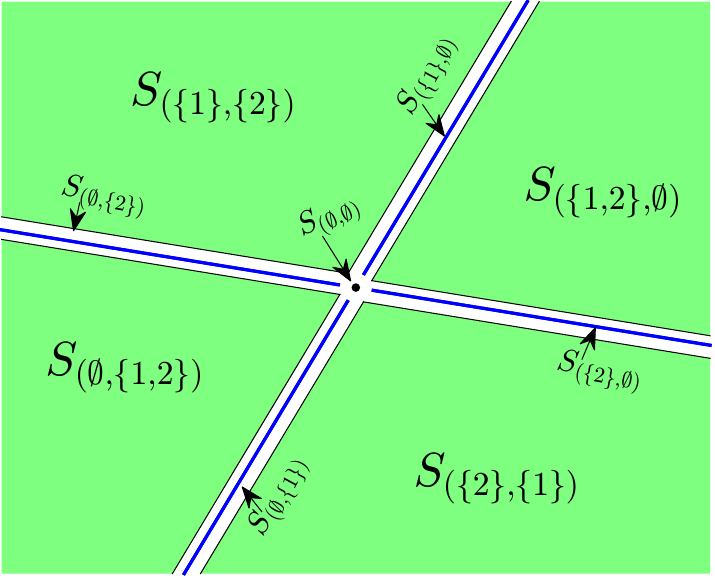}
\caption{ }
\end{subfigure}
\hfill
\caption{\emph{Partition Description.}
\textbf{(a)} Two hyperplanes $P_1$ and $P_2$ in $\IR^2$, with normals $a_1$ and $a_2$ respectively, intersect at a point $x_0$. The corresponding dual vectors $a^*_1$ and $a^*_2$ generated by $P_1$ and $P_2$ are shown. \textbf{(b)} An illustration of the partition $\mcS$ of $\IR^2$ generated by $a^*_1$, $a^*_2$ and $x_0$. There are four 2-dimensional sets (the green regions), four 1-dimensional sets (the blue rays), and one 0-dimensional set (the black point $x_0$). We have added white space between the sets to stress that they are pair-wise disjoint.}
\label{fig:dualbasispartition}
\end{figure}

\paragraph{Codomain Partition.}
We will now introduce another partition of $\IR^d$ that will be useful in describing the geometrical structure of the codomain of the ReLU layer.
Since $x_0$ and $\{a_i^*: i\in I\}$ are derived from the parameters $A$ and $b$ in \eqref{eq:map}, it is easy to verify that in the canonical case, when $A=I_d$ (the $d\times d$ identity matrix) and $b=0$, we have $x_0=0$ and $a^*_i=e_i$. We will use the hat symbol to denote sets defined by \eqref{eq:S} in this specific setting, that is
\begin{align}
\widehat{S}_{(I_+,I_-)} = \bigg\{x\in \IR^d:x=\sum_{i \in I_+} \beta_i e_i - \sum_{i \in I_-} \beta_i e_i, \text{ with } \beta_i>0\bigg\}
\label{eq:Shat}
\end{align}
For instance, we have $\widehat{S}_{(I,\emptyset)}=\IR^d_{++}$ --- the strictly positive orthant in $\IR^d$, and $\widehat{S}_{(\emptyset,\emptyset)}=\{0\}$. In accordance with the domain partition above, the family of sets on the form \eqref{eq:Shat} also generates a partition of $\IR^d$, which we denote by $\widehat{\mcS}$. Examples of partitions $\widehat{\mcS}$ and $\mcS$ of $\IR^3$ are depicted in Figure~\ref{fig:S}.

\paragraph{Affine Equivalence.}
The affine map $A_b:\IR^d\rightarrow \IR^d$ induces a one-to-one correspondence between the sets in $\mcS$ and those in $\widehat{\mcS}$ in the sense that for all $\bfI \in \mcI$, the restriction $A_b:S_{\bfI}\rightarrow \widehat{S}_{\bfI}$ is a bijection. Indeed, for any $x\in S_{\bfI}=S_{(I_+,I_-)}$ we have that $x=x_0+\sum_{i\in I_+} \alpha_i a_i^* - \sum_{i\in I_-} \alpha_i a_i^*$ for some $\alpha_i>0$ by \eqref{eq:S}. From equation \eqref{eq:affine} it follows
\begin{align}
Ax+b=\sum_{i\in I_+} \alpha_i e_i- \sum_{i\in I_-} \alpha_i e_i \in \widehat{S}_{(I_+,I_-)}
\label{eq:affinemap}
\end{align}
Further, the inverse $A^{-1}_b:\widehat{S}_{\bfI}\rightarrow S_{\bfI}$ is given by $A_b^{-1}(y)=A^{-1}(y-b)$ and for every $y\in \widehat{S}_{(I_+,I_-)}$ with $y=\sum_{i\in I_+} \beta_i e_i - \sum_{i\in I_-} \beta_i e_i$ for some $\beta_i>0$ it holds 
\begin{align}
A^{-1}(y-b)=x_0+\sum_{i\in I_+} \beta_i a^*_i- \sum_{i\in I_-} \beta_i a_i^*\in S_{(I_+,I_-)}
\label{eq:inverseaffinemap}
\end{align}
because of the definition of $x_0$ and recalling that $a_i^*$ is the $i$:th column vector in $A^{-1}$. Thus, the affine map takes every set $S_{\bfI}$ in $\mcS$ to the corresponding set $\widehat{S}_{\bfI}$ in $\widehat{\mcS}$ and vice versa for the inverse. In that sense, the pairings $(I_+,I_-)$ are invariant under the affine map due to our construction of the partitions $\mcS$ and $\widehat{\mcS}$.

\paragraph{Closure and Boundary.}
The closure of a set $S_{(I_+,I_-)}$ is given by
\begin{align}
\overline{S}_{(I_+,I_-)}=\bigg\{x\in \IR^d: x=x_0+\sum_{i\in I_+}\alpha_i a_i^*-\sum_{i\in I_-}\alpha_i a_i^*, \text{ with } \alpha_i\geq 0\bigg\}
\end{align}
i.e., the conical hull of the set $\{a_i^*:i\in I_+\}\cup \{-a_i^*:i\in I_-\}$ translated by $x_0$. In particular, $\overline{S}_{(I,\emptyset)}$ is a polyhedral cone with apex $x_0$ with supporting hyperplanes $P_i$, $i\in I$, as illustrated in Figure~\ref{fig:cones}. Therefore, in the canonical case we get $\overline{\widehat{S}}_{(I,\emptyset)}=\IR^d_+$. We can also express the closure in terms of other sets in $\mcS$. To see this, we introduce a partial order $\preceq$ on $\mcI$. For $\bfI,\bfJ \in \mcI$ with $\bfI=(I_+,I_-)$, $\bfJ=(J_+,J_-)$ we define $\bfJ \preceq \bfI$ if and only if $J_+ \subseteq I_+$ and $J_- \subseteq I_-$. Further, we write $\bfJ \prec \bfI$ if $\bfJ\preceq \bfI$ and $\bfJ\neq \bfI$. This gives the following compact expressions for the closure respectively the boundary of a set $S_{\bfI}$
\begin{align}
\overline{S}_{\bfI}=\bigcup_{\bfJ \preceq \bfI}S_{\bfJ}
\qquad\text{and}\qquad
\partial S_{\bfI}=\bigcup_{\bfJ \prec \bfI}S_{\bfJ}
\label{eq:closure}
\end{align}

\begin{figure} 
\centering
\begin{subfigure}[b]{0.35\textwidth}
\centering
\includegraphics[width=\textwidth]{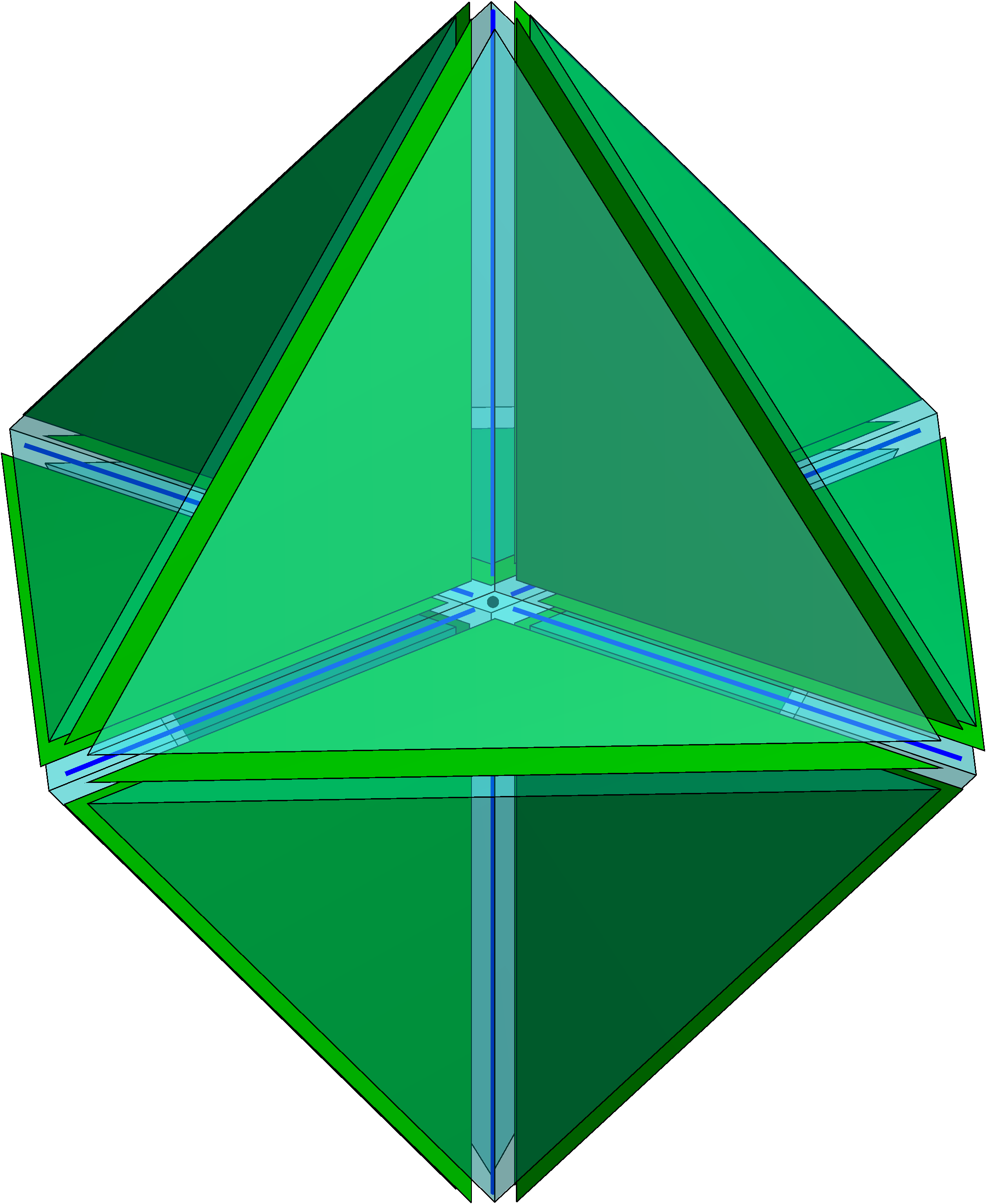}
\caption{ }
\end{subfigure}
\quad
\begin{subfigure}[b]{0.45\textwidth}
\centering
\includegraphics[width=\textwidth]{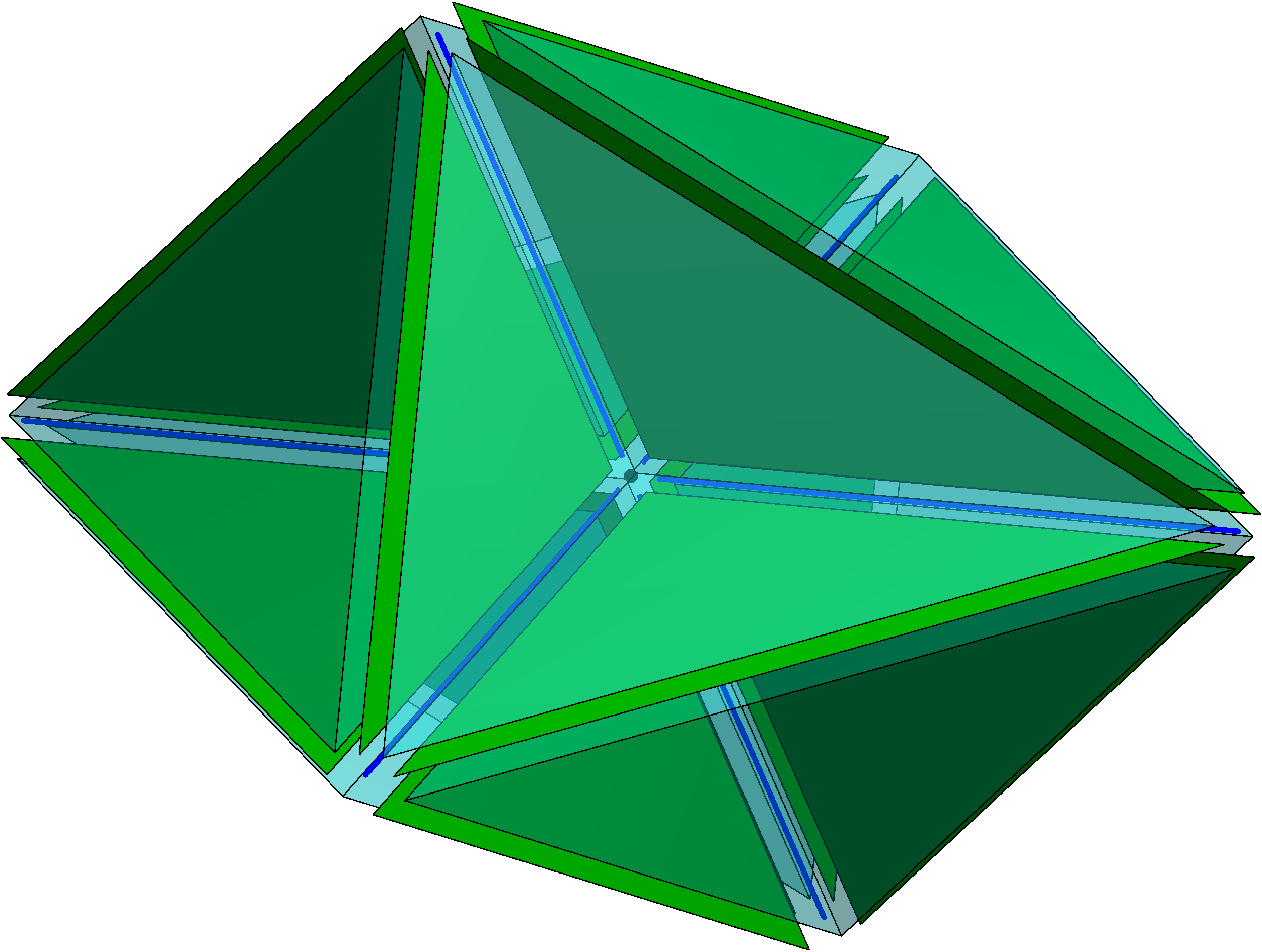}
\caption{ }
\end{subfigure}
\hfill
\caption{\emph{Domain and Codomain Partitions.}
In case $d=3$ the families $\widehat{\mcS}$ and $\mcS$ will both partition $\IR^3$ in eight 3-dimensional sets (the transparent volumes), twelve 2-dimensional sets (the green faces), six 1-dimensional sets (the blue rays) and one 0-dimensional set (the black point). In fact, the sets extend outwards from $x_0$ (the black point in the center) infinitely, but for illustrative purposes only slices of them are shown. We have also intentionally added space between the sets to stress that they are pairwise disjoint and make distinguishing them easier. \textbf{(a)} Illustration of the canonical partition $\widehat{\mcS}$. \textbf{(b)} Illustration of $\mcS$ where $a_1^*=\frac{e_1}{4}-e_2+\frac{e_3}{10}$, $a_2^*=e_1+\frac{e_3}{4}$, $a_3^*=-\frac{e_2}{2}+e_3$ and $x_0=[1,1,1]^T$.  Every $S_{\bfI}\in \mcS$ is the preimage of the corresponding set $\widehat{S}_{\bfI}\in \widehat{\mcS}$ under the affine map $x\mapsto Ax+b$ and vice versa.}
\label{fig:S}
\end{figure}

\begin{figure} 
\centering
\begin{subfigure}[b]{0.49\textwidth}
\centering
\includegraphics[width=\textwidth]{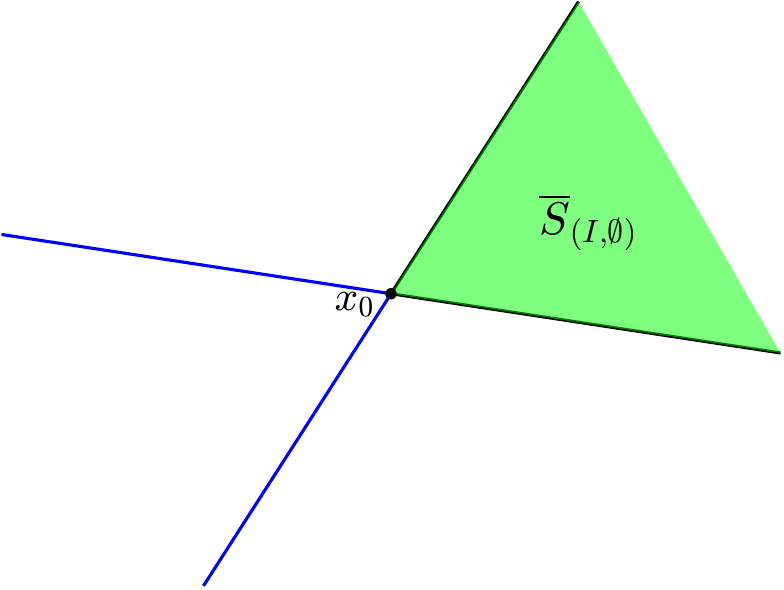}
\caption{ }
\end{subfigure}
\quad
\begin{subfigure}[b]{0.45\textwidth}
\centering
\includegraphics[width=\textwidth]{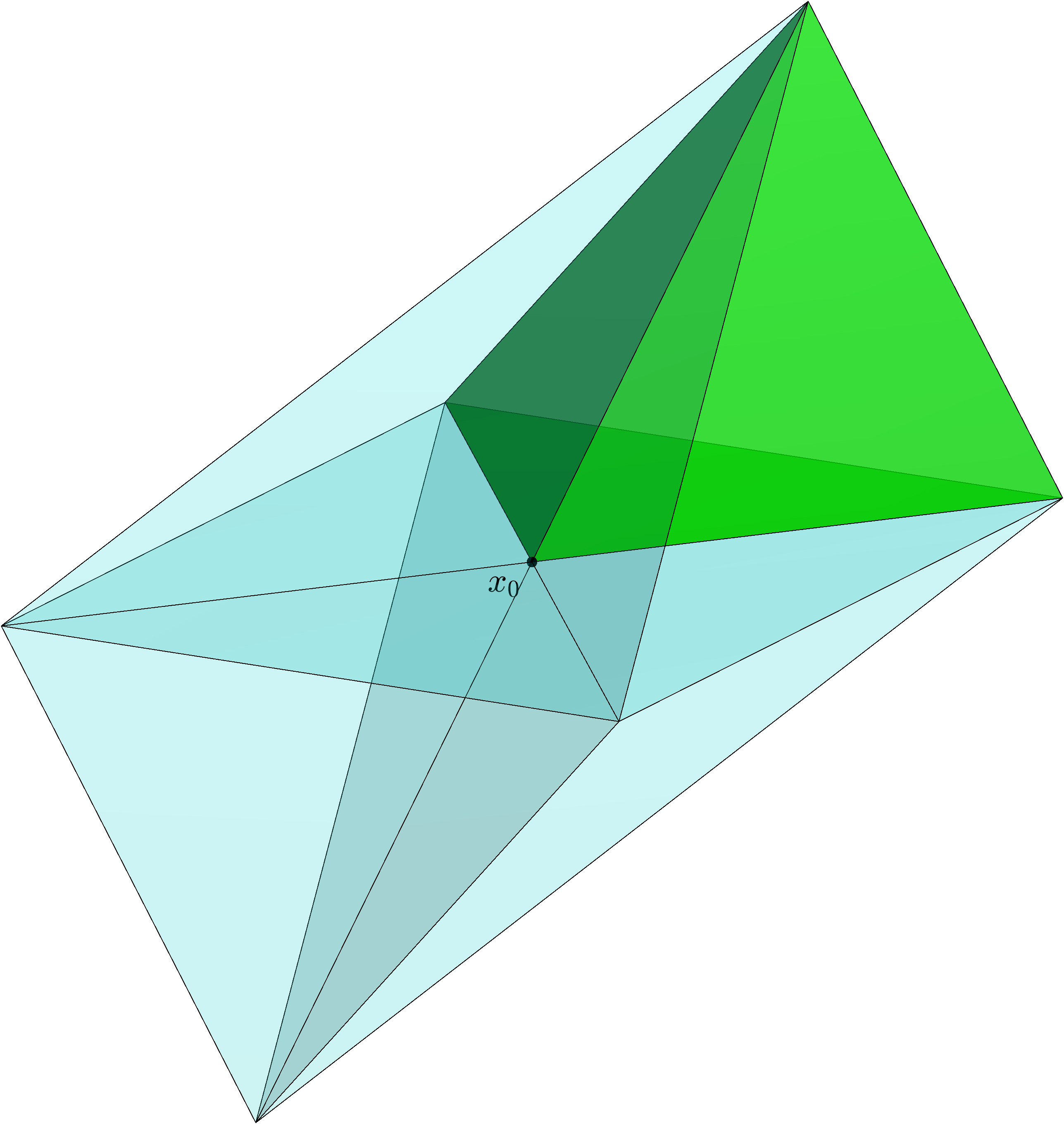}
\caption{ }
\end{subfigure}
\hfill
\caption{\emph{Polyhedral Cone.}
Illustrations of the polyhedral cone $\overline{S}_{(I,\emptyset)}$ (in green), with apex at $x_0$, for \textbf{(a)} the partition in Figure~\ref{fig:dualbasispartition} and \textbf{(b)} the partition in Figure~\ref{fig:S} (viewed from the apex of the cone). Only slices of these cones are shown since both of them extend indefinitely.}
\label{fig:cones}
\end{figure}

\subsection{Image of a ReLU Layer}
When considering the entire map $T$, the situation is slightly more complicated due to the application of the ReLU function. In general, several sets in $\mcS$ will be mapped to the same set in $\widehat{\mcS}$. As we will see the boundary $\partial \widehat{S}_{(I,\emptyset)}$ of $\widehat{S}_{(I,\emptyset)} = \IR^d_+$ given by
\begin{align}
\partial \widehat{S}_{(I,\emptyset)} =\bigcup_{\bfJ \prec  (I,\emptyset)}\widehat{S}_{\bfJ}= \bigcup_{J \subset I} \widehat{S}_{(J,\emptyset)}
\end{align}
will be an important object when studying the structure of $T$ and therefore we also introduce the partition $\partial \widehat{\mcS}=\{\widehat{S}_{\bfJ}:\bfJ\prec (I,\emptyset)\}$ of this boundary. Along the same lines, we define $\partial \mcS$. Lemma~\ref{lemma:map} below describes the action of $T$ on sets in $\mcS$. 
\begin{lem}[Image Structure of a ReLU Layer]
\label{lemma:map}
Given a set $S_{(I_+,I_-)}\in \mcS$ it holds
\begin{align}
\boxed{
T(S_{(I_+,I_-)})=\widehat{S}_{(I_+,\emptyset)}
}
\end{align}
\end{lem}
\begin{proof}
Consider a point $x\in S_{(I_+,I_-)}$, which by \eqref{eq:S} has the expansion
\begin{align}
x=x_0+\sum_{i\in I_+} \alpha_i a_i^* - \sum_{i\in I_-} \alpha_i a_i^*
\end{align}
with coefficients $\alpha_i>0$. Applying the ReLU layer \eqref{eq:map} and using equation \eqref{eq:affinemap} we get
\begin{align}
\label{eq:TonSector}
T(x)&=\mathrm{ReLU}(Ax+b,0)
\\&
=\max\bigg(\sum_{i\in I_+} \alpha_i e_i- \sum_{i\in I_-} \alpha_i e_i,0\bigg)
\\&
=\sum_{i\in I_+}\alpha_ie_i \in \widehat{S}_{(I_+,\emptyset)}
\label{eq:TonSector-last}
\end{align}
where the last equality holds by the definition of the Euclidian basis vectors $e_i$ and that $\alpha_i > 0$.
Conversely, for any $y=\sum_{i\in I_+} \beta_i e_i \in  \widehat{S}_{(I_+,\emptyset)}$ where $\beta_i>0$ we have
\begin{align}
A^{-1}(y-b)-\sum_{i\in I_-}\alpha_ia^*_i \in S_{(I_+,I_-)}
\end{align}
for any choice of $\alpha_i>0$ according to equation \eqref{eq:inverseaffinemap}. Also,
\begin{align}
T\bigg(A^{-1}(y-b)-\sum_{i\in I_-}\alpha_ia^*_i\bigg)=y
\label{eq:Ty}
\end{align}
and hence for all $y\in \widehat{S}_{(I_+,\emptyset)}$ there is an $x\in S_{(I_+,I_-)}$ such that $T(x)=y$.
\end{proof}
Lemma~\ref{lemma:map} reveals that whenever $I_+ \subset I$ we have 
\begin{align}
T(S_{(I_+,I_-)})=\widehat{S}_{(I_+,\emptyset)}\in \partial\widehat{\mcS}
\end{align}
and in particular when $I_+=I$ we get
\begin{align}
T(S_{(I,\emptyset)})=\widehat{S}_{(I,\emptyset)}
\end{align}
Hence, $T$ reduces to the affine map $A_b$ when restricted to $S_{(I,\emptyset)}$ and therefore we will refer to $S_{(I,\emptyset)}$ as the affine sector in $\mcS$. In fact, $T$ acts affinely on all points in the polyhedral cone $\overline{S}_{(I,\emptyset)}$. Moreover, by the lemma we also see that 
\begin{align}
\dim(T(S_{(I_+,I_-)}))=\dim(\widehat{S}_{(I_+,\emptyset)})=|I_+|\leq |I_+\cup I_-|=\dim(S_{(I_+,I_-)})
\end{align}
Hence, the dimension of the images of $S_{(I_+,I_-)}$ with $I_-\neq \emptyset$ under the map $T$ is reduced. The only sets in $\mcS$ with preserved dimension are those that are subsets of $\overline{S}_{(I,\emptyset)}$. Especially, the only set with a non-zero measure in $\IR^d$ for which the dimension is preserved is the affine sector $S_{(I,\emptyset)}$ on which $T$ acts affinely. All other sets in $\mcS$ with non-zero measure will be mapped onto some lower dimensional set in $\partial \widehat{\mcS}$. Since points in a dataset will generally belong to a subset of the $d$-dimensional sets (those with non-zero measure) in $\mcS$ and all of them, but one, will be mapped to some lower dimensional set on the boundary $\partial \widehat{S}_{(I,\emptyset)}$ it is clear that $T$ has contracting properties. Thus, iteratively applying maps of the form \eqref{eq:map} as is done in deep fully-connected ReLU networks will efficiently contract the input data.

\paragraph{Geometric Interpretation.}
The ReLU layer \eqref{eq:map} is constructed as the affine map $A_b$ followed by the ReLU activation function, which is a projection $\IR^d \mapsto \overline{\widehat{S}}_{I,\emptyset} = \IR^d_+$.
Using the geometric structure defined above, we will now give an alternative construction where the order of these operations is reversed, as a projection onto a polyhedral cone followed by the affine map $A_b$. We know that a point $x\in S_{(I_+,I_-)}$ with expansion $x=x_0+\sum_{i\in I_+} \alpha_i a_i^* - \sum_{i\in I_-} \alpha_i a_i^*$, $\alpha_i>0$, is mapped to the point $y=\sum_{i\in I_+} \alpha_i e_i\in \widehat{S}_{(I_+,\emptyset)}\subset \IR^d_+$. We can split this transformation into two steps
\begin{align}
\underbrace{x_0+\sum_{i\in I_+} \alpha_i a_i^* - \sum_{i\in I_-} \alpha_i a_i^*}_{\in S_{(I_+,I_-)}} \stackrel{\pi} \longmapsto \underbrace{x_0+\sum_{i\in I_+} \alpha_i a_i^*}_{\in S_{(I_+,\emptyset)}\subset \overline{S}_{(I,\emptyset)}} \stackrel{A_b} \longmapsto \underbrace{\sum_{i\in I_+} \alpha_i e_i}_{\in\widehat{S}_{(I_+,\emptyset)}\subset\IR^d_+} 
\label{eq:decomposition}
\end{align}
This suggests that we can decompose the ReLU layer $T:\IR^d\rightarrow \IR^d_+$ as
\begin{align} \label{eq:commutating}
\boxed{
T =\mathrm{ReLU}\circ A_b
= A_b\circ \pi
}
\end{align}
where $\pi:\IR^d\rightarrow \overline{S}_{(I,\emptyset)}$ is a surjective projection mapping the input space $\IR^d$ onto the polyhedral cone $\overline{S}_{(I,\emptyset)}$ and $A_b:\overline{S}_{(I,\emptyset)}\rightarrow \IR^d_+$ is the bijective affine map given by \eqref{eq:A_b} mapping the polyhedral cone onto the non-negative orthant $\IR^d_+$.
Note that \eqref{eq:commutating} gives a description of how $A_b$ commutes with the ReLU.
The projection $\pi$ is piecewise defined on the sectors in $\mcS$. For a point $x\in S_{(I_+,I_-)}$ 
we define
\begin{align}
\pi(x)=\pi\bigg(x_0+\sum_{i\in I_+} \alpha_i a_i^* - \sum_{i\in I_-} \alpha_i a_i^*\bigg)=x_0+\sum_{i\in I_+} \alpha_i a_i^*
\label{eq:pi_def}
\end{align}
which clearly is a projection since $\pi\circ \pi(x) = \pi(x)$. For points $x\in \overline{S}_{(I,\emptyset)}$ the projection $\pi$ acts trivially, i.e., $\pi(x)=x$, whereas points outside the cone will be mapped to some part of the cone boundary. Apart from a translation, the non-linear properties of $T$ are entirely captured by the projection $\pi$ since $A_b$ is affine. The geometrical structure of the decomposition of $T$ is depicted in Figure~\ref{fig:composition}.
\begin{figure}
\centering
\includegraphics[width=\textwidth]{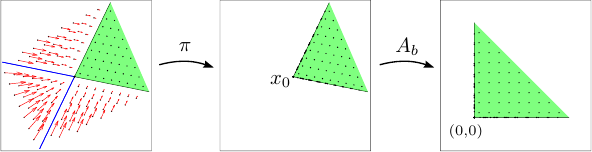}
\caption{\emph{Geometrical Description of a ReLU layer.}
An illustration in $\IR^2$ of the geometrical structure of $T$ written as a composition of a projection $\pi$ mapping $\IR^d$ onto the cone $\overline{S}_{(I,\emptyset)}$ followed by $A_b$ mapping $\overline{S}_{(I,\emptyset)}$ onto $\IR^d_+$ affinely.}
\label{fig:composition}
\end{figure}
From the construction of $\pi$ it is clear that a sector $S_{(I_+,I_-)}$ with $I_-\neq \emptyset$ will be projected onto $S_{(I_+,\emptyset)} \in \partial \mcS$. Hence, the $|I_+\cup I_-|$-dimensional set $S_{(I_+,I_-)}$ is projected onto a $|I_+|$-dimensional part of the boundary of the polyhedral cone. Moreover, from the definition \eqref{eq:pi_def} of $\pi$ we see that the projection is parallel to the subspace span$\{a_i^*:i\in I_-\}$. The action of $\pi$ in $\IR^3$ is illustrated in Figure~\ref{fig:projection3D}.

\begin{figure}
\centering
\includegraphics[width=0.5\linewidth]{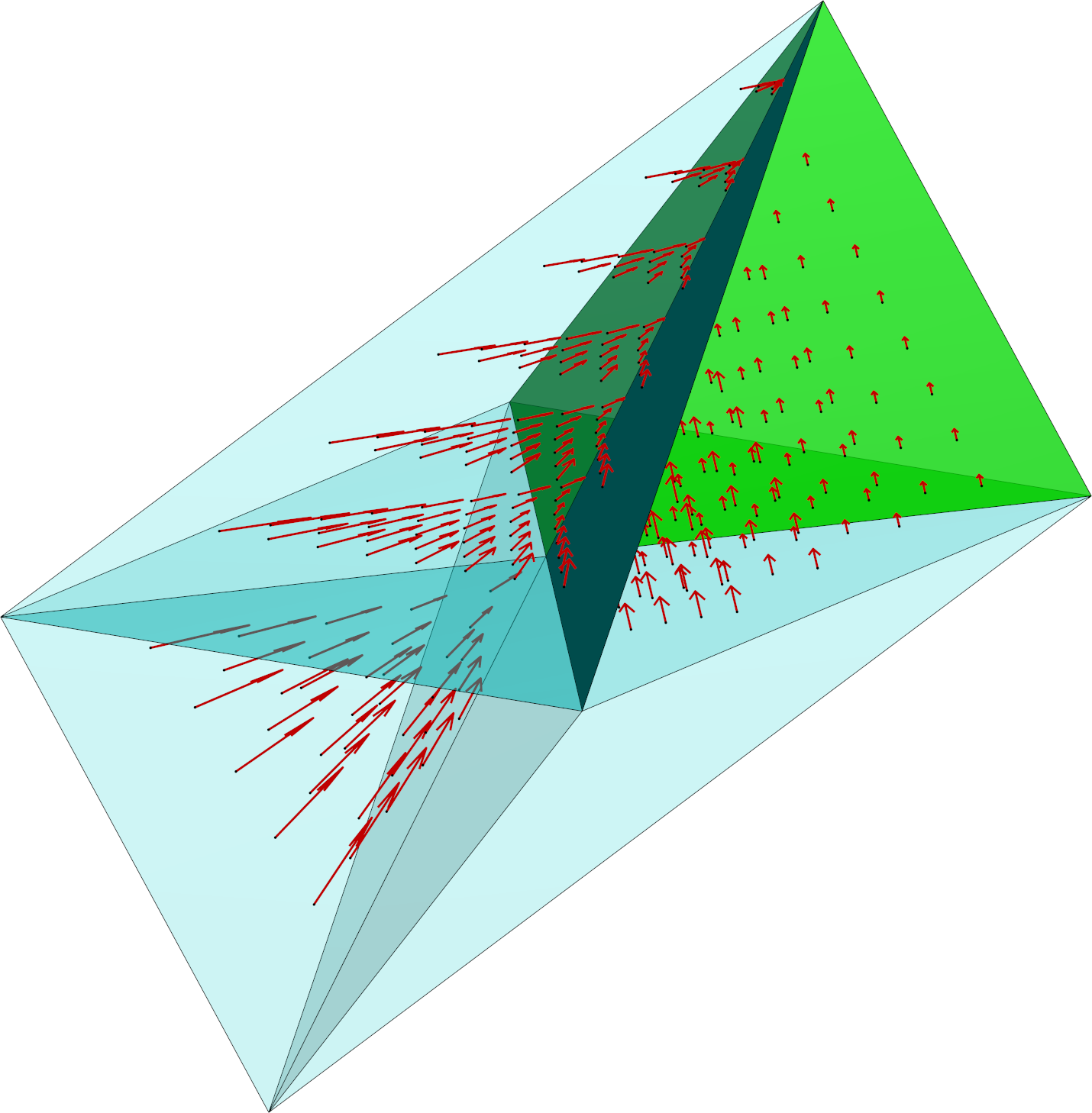}
\caption{\emph{Polyhedral Cone Projection.}
The projection $\pi$ will map $\IR^d$ onto the polyhedral cone $\overline{S}_{(I,\emptyset)}$. A sector $S_{(I_+,I_-)}$ in $\mcS$ will be projected onto a $|I_+|$-dimensional part of the boundary of the cone. The figure illustrates how three different 3-dimensional sectors are projected onto the boundary of a cone in $\IR^3$. Points in the upper right sector are projected onto a 2-dimensional face of the cone along one of the dual vectors. Points in the upper left sector will be projected onto the 1-dimensional edge of the boundary. That projection is parallel with a plane spanned by two of the dual vectors. The bottom left sector, opposite to the interior of the cone, will be mapped to the 0-dimensional apex $x_0$.}
\label{fig:projection3D}
\end{figure}

\begin{rem}[Contracting ReLU Layers] \label{rem:contracting}
By a minor modification, this geometrical description also extends to ReLU layers where the input dimension is reduced, i.e., $T:\IR^d\rightarrow \IR^m_+$ where $d>m$. In this scenario, $A\in \IR^{m\times d}$ will have fewer rows than columns and we only get $m$ hyperplanes in $\IR^d$, defined as in \eqref{eq:hyppp}. Assuming the rows of $A$ are linearly independent, the hyperplanes will not intersect in a point but in a $(d-m)$-dimensional affine subspace $\bigcap_{i=1}^m P_i$ of $\IR^d$. However, if we define the $m$-dimensional subspace $V=\text{span}(\{a_1,a_2,\hdots,a_m\})$ we get that $V \cap \bigl(\bigcap_{i=1}^m P_i \bigr)$ is a single point $x_0 \in V$. Similarly, if we let $I=\{1,2,\hdots,m\}$ we can, for each $i\in I$, define the dual vector $a_i^*$ parallel to the line 
\begin{align}
L_i=V \cap \Bigl(\bigcap_{j\in I\setminus \{i\}} P_j \Bigr)
\end{align}
and scaled such that $a_i\cdot a_j^*=\delta_{ij}$ for $i,j\in I$. These $m$ dual vectors will also be linearly independent and therefore they will be a basis of the subspace $V$. Using these dual vectors, we proceed as before by defining a partition of the subspace $V$ using the sets
\begin{align}
S_{(I_+,I_-)}= \bigg\{ x\in \IR^d \,:\, x=x_0+\sum_{i \in I_+} \alpha_i a_i^* - \sum_{i \in I_-} \alpha_i a_i^*, \ \alpha_i>0\bigg\}
\end{align}
for disjoint index sets $I_+, I_- \subseteq I=\{1,2,\hdots,m\}$. In this case we get that $\overline{S}_{(I,\emptyset)}\subset V$ is an $m$-dimensional cone with apex at $x_0$ embedded in $\IR^d$.
Let $V^\perp$ be the orthogonal complement to $V$ in $\IR^d$, i.e.,
\begin{align} \label{eq:orthogonalV}
V^\perp = \{ x \in \IR^d \,:\, x \cdot v = 0, \ \forall v \in V \}
\end{align}
and let $\{w_i\}_{i\in I_\perp}$, $I_\perp=\{m+1,\dots,d\}$, be some basis to $V^\perp$.
Since $\IR^d = V \oplus V^\perp$ and $x_0 \in V$, each point $x\in \IR^d$ has the expansion
\begin{align}
x=
\underbrace{x_0+\sum_{i\in I} \lambda_i a^*_i}_{\in V}
+
\underbrace{\sum_{i \in I_\perp}\lambda_i w_i}_{\in V^\perp} ,
\qquad\text{where $\lambda_i\in \IR$}
\end{align}
Because the rows of $A$ are vectors in $V$ whereas $w_i \in V^\perp$ we by the definition of the orthogonal complement \eqref{eq:orthogonalV} have $Aw_i=0$,
and it follows that $T$ is invariant to components in $V^\perp$ such that 
\begin{align}
T(x) &=
T\bigg(x_0+\sum_{i\in I}\lambda_ia^*_i + \sum_{i\in I_\perp}\lambda_iw_i\bigg)
=
T\bigg(x_0+\sum_{i\in I} \lambda_ia^*_i\bigg)
\end{align}
We can incorporate this into our geometric description by prepending an orthogonal projection onto the subspace $V \subset \IR^d$.
This gives a decomposition of $T:\IR^d\rightarrow \IR^m_+$ as
\begin{align}
\begin{split}
&x=\underbrace{x_0+\sum_{i\in I_+} \alpha_i a_i^* - \sum_{i\in I_-} \alpha_i a_i^*+  \sum_{i\in I_\perp} \lambda_iw_i}_{\in \IR^d}\\
& \stackrel{P_V}\longmapsto \underbrace{x_0+\sum_{i\in I_+} \alpha_i a_i^* - \sum_{i\in I_-} \alpha_i a_i^*}_{\in  S_{(I_+,I_-)}\subset V} \stackrel{\pi} \longmapsto \underbrace{x_0+\sum_{i\in I_+} \alpha_i a_i^*}_{\in S_{(I_+,\emptyset)}\subset \overline{S}_{(I,\emptyset)}} \stackrel{A_b} \longmapsto \underbrace{\sum_{i\in I_+} \alpha_i e_i}_{\in \overline{\widehat{S}}_{(I,\emptyset)} = \IR^m_+}
\end{split}
\end{align}
where $P_V:\IR^d\rightarrow V$ is the orthogonal projection onto $V$, 
whereafter the same geometrical description \eqref{eq:decomposition} as in the case of preserved dimension is used with the difference that the cone projection takes place in the $m$-dimensional subspace $V\subset \IR^d$.
\end{rem}

\subsection{Preimage of a ReLU Layer}
The preimage of a set $\widehat{\omega}\subseteq \IR^d$ under the ReLU layer $T$ is the set of all elements in the domain $\IR^d$ that $T$ maps into $\widehat{\omega}$, and we denote the preimage by $T^{-1}(\widehat{\omega})$. 
Based on Lemma~\ref{lemma:map} and its proof, we will here express preimages under $T$ using the geometrical structure of the domain and codomain detailed above.
Firstly, since $T$ maps $\IR^d$ onto $\IR^d_+=\overline{\widehat{S}}_{(I,\emptyset)}$, the preimage of any point $y \notin \overline{\widehat{S}}_{(I,\emptyset)}$ will be empty. 
Secondly, equation \eqref{eq:Ty} in the proof of Lemma~\ref{lemma:map} shows that the set 
\begin{align}
\bigg\{x\in \IR^d \,:\, x=A^{-1}(y-b)-\sum_{i\in I_-}\alpha_ia^*_i, \,\, \alpha_i>0\bigg\}
\label{eq:preimsector}
\end{align}
contains all points in $S_{(I_+,I_-)}$ that are mapped to a specific point $y\in \widehat{S}_{(I_+,\emptyset)} \subset \overline{\widehat{S}}_{(I,\emptyset)}$. In other words, the set in \eqref{eq:preimsector} contains the preimage of $y$ intersected with $S_{(I_+,I_-)}$.
Thirdly, Lemma~\ref{lemma:map} reveals that the ReLU layer is invariant with respect to $I_-$, i.e.,
\begin{align}
T(S_{(I_+,I_-)})=T(S_{(I_+,J_-)}) \qquad \forall J_- \subseteq I \setminus I_+
\label{eq:equalimage}
\end{align}
Hence, the complete preimage of a single point $y\in \widehat{S}_{(I_+,\emptyset)}$ is given by
\begin{align}
T^{-1}(y)&=\bigcup_{K\subseteq I\setminus I_+} \bigg\{x\in \IR^d \,:\, x=A^{-1}(y-b)-\sum_{i\in K}\alpha_ia^*_i, \,\, \alpha_i>0\bigg\}
\\&=
\bigg\{x\in \IR^d \,:\, x= A^{-1}(y-b)-\sum_{i\in I\setminus I_+}\alpha_ia^*_i, \,\, \alpha_i\geq 0\bigg\}
\label{eq:preimy}
\end{align}
If $y\in \widehat{S}_{(I_+,\emptyset)}$, its complete preimage is spanned by $\{a_i^*:i\in I\setminus I_+\}$, so
\begin{align}
\dim(T^{-1}(y))=d-\dim(\widehat{S}_{(I_+,\emptyset)})=\text{codim}(\widehat{S}_{(I_+,\emptyset)})
\end{align}
Hence, points intersecting a lower dimensional facet of $\partial \widehat{\mcS}$ will generate a preimage of higher dimension than points intersecting a higher dimensional facet. Figure~\ref{fig:preimy} illustrates the structure of the preimages for different points in $\IR^3_+$. The preimage of an entire set $\widehat{S}_{\bfI} \in \widehat{\mcS}$ is expressed in the following lemma.
\begin{figure}
\centering
\begin{subfigure}[b]{0.33\textwidth}
\centering
\includegraphics[width=\textwidth]{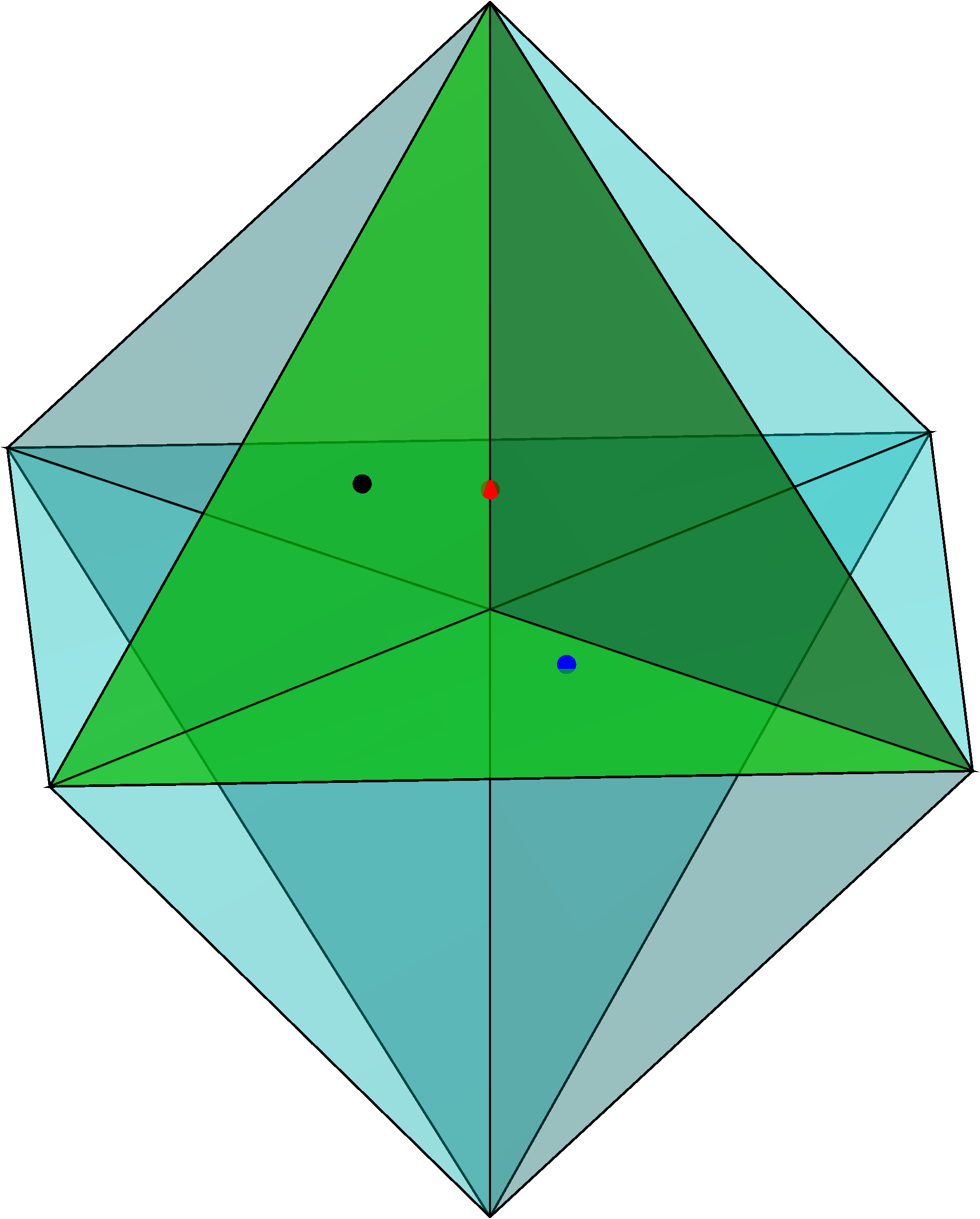}
\caption{ }
\end{subfigure}
\quad
\begin{subfigure}[b]{0.45\textwidth}
\centering
\includegraphics[width=\textwidth]{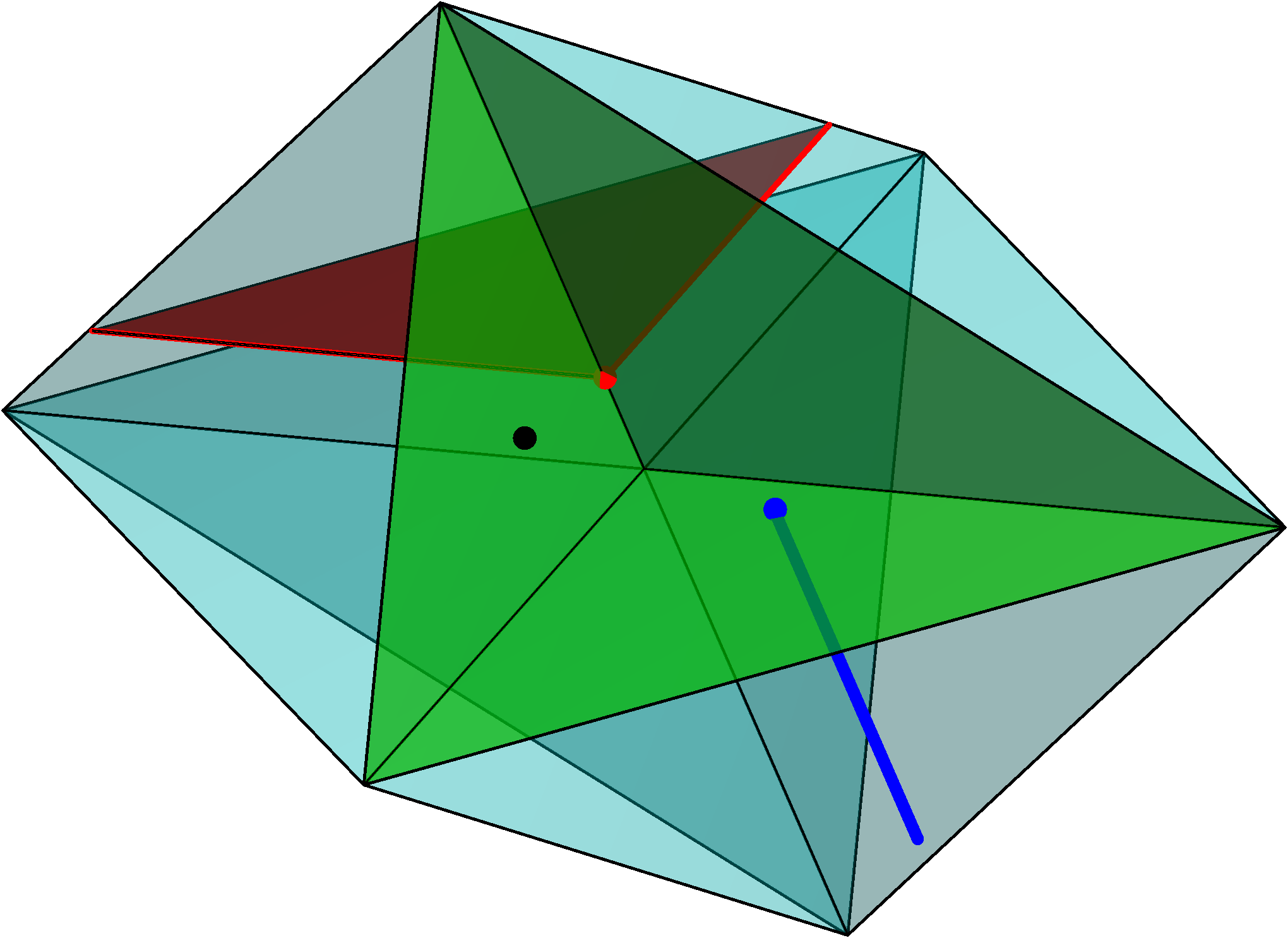}
\caption{ }
\end{subfigure}
\hfill
\caption{\emph{Preimages of Points.}
Examples of preimages of three different points $y_1$, $y_2$ and $y_3$ in $\IR^3_+$. The same partitions $\widehat{\mcS}$ and $\mcS$ as in Figure~\ref{fig:S} are illustrated with all but the two-dimensional sets removed. The facets of the boundary of $\IR^d_+$ and of the polyhedral cone $\overline{S}_{(I,\emptyset)}$ are colored green. \textbf{(a)} The partition $\widehat{\mcS}$ with a point $y_1\in \widehat{S}_{(\{1,2,3\}, \emptyset)}$ (in black), a point $y_2 \in \widehat{S}_{(\{1,2\}, \emptyset)}$ (in blue) and a point $y_3 \in \widehat{S}_{(\{3\}, \emptyset)}$ (in red) are shown. \textbf{(b)} The partition $\mcS$ with the points $x_i=A_b^{-1}(y_i)$, $i\in \{1,2,3\}$ marked, $x_1 \in S_{(\{1,2,3\}, \emptyset)}$ (in black), $x_2 \in S_{(\{1,2\}, \emptyset)}$ (in blue) and $x_3 \in \widehat{S}_{(\{3\}, \emptyset)}$ (in red). The preimage for each $y_i$ is shown with the corresponding color and whose dimension is the codimension of the set in $\widehat{\mcS}$ to which $y_i$ belongs.}
\label{fig:preimy}
\end{figure}
\begin{lem}[Preimage Structure under a ReLU Layer]
\label{lemma:preim}
The preimage of the set $\widehat{S}_{(J,\emptyset)}\in \widehat{\mcS}$ under $T$, where $J\subseteq I$, is given by
\begin{align}
\boxed{
T^{-1}(\widehat{S}_{(J,\emptyset)})=\bigcup_{K \subseteq I\setminus J} S_{(J,K)}
}
\end{align}
\end{lem}
\begin{proof}
Starting from \eqref{eq:preimy}, the preimage of the entire set $\widehat{S}_{(J,\emptyset)}$ can computed as
\begin{align}
T^{-1}(\widehat{S}_{(J,\emptyset)})&=\bigcup_{y\in \widehat{S}_{(J,\emptyset)}}\bigg\{x\in \IR^d \,:\, x= A^{-1}(y-b)-\sum_{i \in I \setminus J} \alpha_i a_i^*, \,\, \alpha_i \geq 0 \bigg\}\\
&=\bigcup_{x'\in S_{(J,\emptyset)}}\bigg\{x\in \IR^d \,:\, x=x'-\sum_{i \in I \setminus J} \alpha_i a_i^*, \,\, \alpha_i \geq 0 \bigg\}\\
&=\bigcup_{x'\in S_{(J,\emptyset)}}\bigcup_{K \subseteq I\setminus J}\bigg\{x\in\IR^d \,:\, x=x'-\sum_{i \in K} \alpha_i a_i^*, \,\, \alpha_i > 0 \bigg\}\\
&=\bigcup_{K \subseteq I\setminus J}\bigcup_{x'\in S_{(J,\emptyset)}}\bigg\{x\in\IR^d \,:\, x=x'-\sum_{i \in K} \alpha_i a_i^*, \,\, \alpha_i > 0 \bigg\}
\\&
=\bigcup_{K \subseteq I\setminus J} S_{(J,K)}
\end{align}
which concludes the proof.
\end{proof}

Using Lemma~\ref{lemma:preim} we can also derive expressions for preimages of the closures of sets in $\widehat{\mcS}$. By the definition of the closure \eqref{eq:closure}, we obtain
\begin{align}
\begin{split}
T^{-1}\big(\overline{\widehat{S}}_{(J,\emptyset)}\big)&=T^{-1}\bigg(\bigcup_{K\subseteq J}\widehat{S}_{(K,\emptyset)}\bigg)=
\bigcup_{K\subseteq J}T^{-1}(\widehat{S}_{(K,\emptyset)})
\\
&=\bigcup_{K\subseteq J}\bigg(\bigcup_{L \subseteq I\setminus K} S_{(K,L)}\bigg)=\bigcup_{K\subseteq J}\overline{S}_{(K,I\setminus K)}
\end{split}
\end{align}

Now, given a subset $\hat{\omega} \subseteq \widehat{S}_{(J,\emptyset)}$ we get
\begin{align}
T^{-1}(\hat{\omega})=\bigg\{x\in \IR^d:x=x'-\sum_{i\in I\setminus J}\alpha_ia_i^*,\, \alpha_i\geq 0,\, x'\in \omega \bigg\}
\label{eq:preimsubset}
\end{align}
where $\omega=A_b^{-1}(\hat{\omega})$, i.e, the preimage of $\hat{\omega}$ under the affine transformation. To keep the notation consistent, we will continue labeling quantities related to sets seen as subsets of the codomain of $T$ using the hat symbol (e.g, $\widehat{S}_{\bfI}$ and $\hat{\omega}$), while the corresponding quantities related to the inverse image of the same subset under the affine map $A_b$ will be labeled in the same way but without the hat (e.g., $S_{\bfI}$ and $\omega$). In the special case when $\hat{\omega} \subseteq \widehat{S}_{(I,\emptyset)}$ the preimage is simply given by
\begin{align}
T^{-1}(\hat{\omega}) =\bigg\{x\in \IR^d \,:\, x= x', \ x' \in \omega \bigg\}=\omega 
\label{eq:preimageaffine}
\end{align}
as $T$ reduces to an invertible affine map on $S_{(I,\emptyset)}$. For a general set $\hat{\omega} \subseteq \IR^d_+$ its preimage under $T$ is completely described in terms of the dual basis $\{ a_i^* : i \in I\}$ and its intersection with the sets in $\partial \widehat{\mcS}$. Geometrically, the preimage of $\hat{\omega}$ is obtained by first mapping the intersection $\hat{\omega} \cap \IR^d_+$ to the cone, i.e., to $\omega \cap \overline{S}_{(I,\emptyset)}$ and then the parts on the boundary of the cone will be extended outwards in directions given by a subset of the dual vectors as illustrated in Figure~\ref{fig:preimset}. 
\begin{figure} 
\centering
\begin{subfigure}[b]{0.33\textwidth}
\centering
\includegraphics[width=\textwidth]{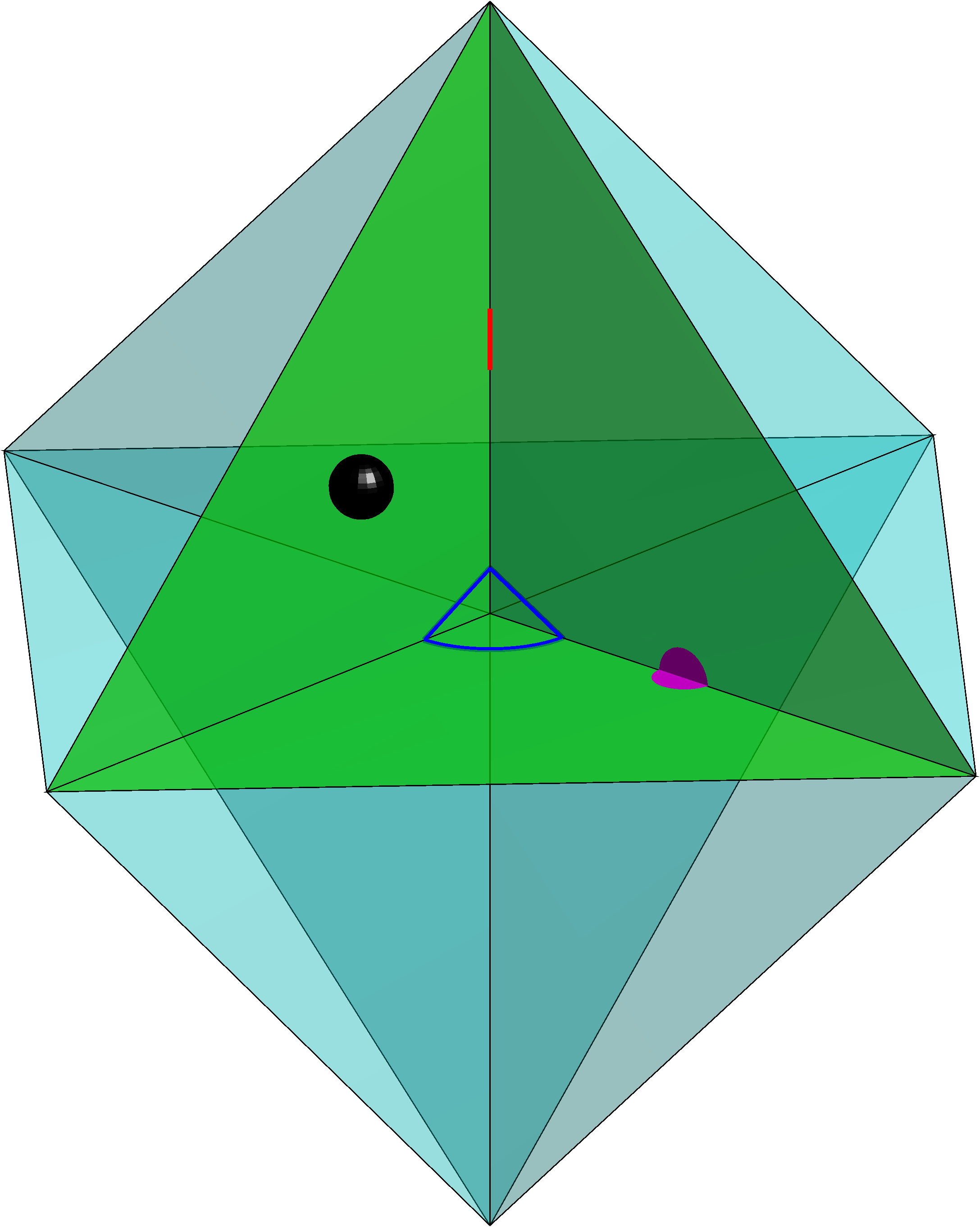}
\caption{ }
\end{subfigure}
\quad
\begin{subfigure}[b]{0.45\textwidth}
\centering
\includegraphics[width=\textwidth]{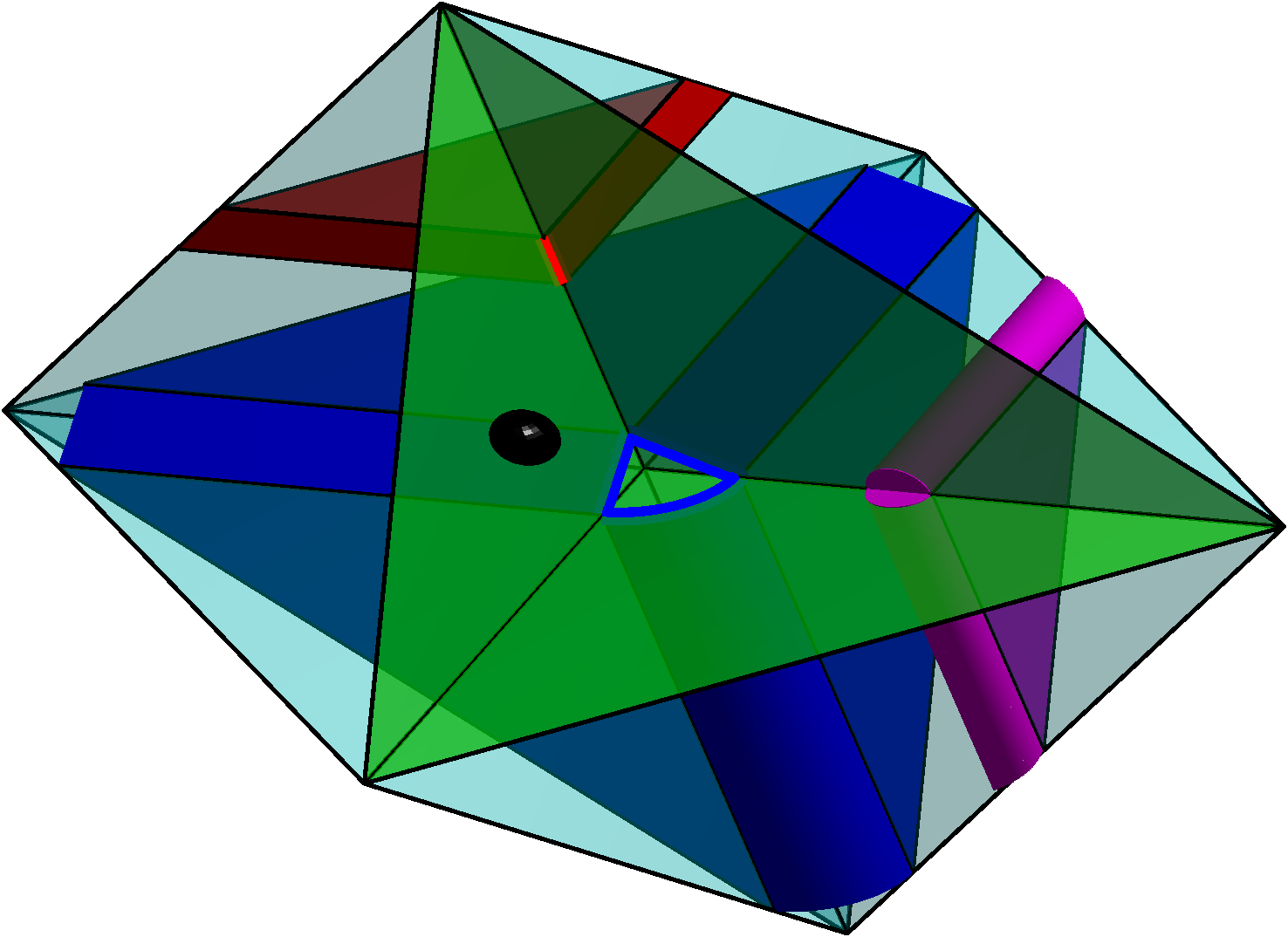}
\caption{ }
\end{subfigure}
\hfill
\caption{\emph{Preimages of Sets.}
Examples of preimages of different sets in $\IR^3_+$. \textbf{(a)} Four sets with different intersections with the sets in $\partial \widehat{\mcS}$. \textbf{(b)} The structure of the preimages depends on how the sets intersect with the boundary. If a set is completely contained in $\widehat{S}_{(I,\emptyset)}$, the preimage is given by the inverse image of the affine transformation (like the black sphere). In other cases, the preimages are spanned by a subset of the dual basis vectors outwards from the inverse image, under the map $x\mapsto Ax+b$, of the intersections with the boundary.}
\label{fig:preimset}
\end{figure}

\section{Application to Feed-Forward Networks}
\label{sec:3}

In this section, we investigate how the results above can be applied in a binary classification setting, where the decision boundary separating two classes is formulated as the zero contour to a feed-forward ReLU network.

\begin{definition}[Fully-Connected ReLU Network]\label{def:network}
A function $F:\IR^{d_{in}}\to \IR^{d_{out}}$ is called a fully-connected ReLU network of widths $d_1,\hdots,d_N$ and depth $N$ if it can be written as the composition
\begin{align}
F(x)=L \circ T^{(N)} \circ \cdots \circ T^{(1)}(x)
\label{eq:deepnet}
\end{align}
where $L:\IR^{d_{N}}\to \IR^{d_{out}}$ is an affine function and $T^{(k)}:\IR^{d_{k-1}}\to \IR^{d_{k}}_+$, $k=1,\hdots, N$, $d_0=d_{in}$ are functions of the form
\begin{align} 
T^{(k)}(x)=\mathrm{ReLU}(A^{(k)}x+b^{(k)})
\label{def:T_n}
\end{align}
parameterized by $A^{(k)}\in \IR^{d_k\times d_{k-1}}$, $b^{(k)}\in \IR^{d_k}$.
\end{definition}
\begin{figure} 
\centering
\includegraphics[width=0.7\linewidth]{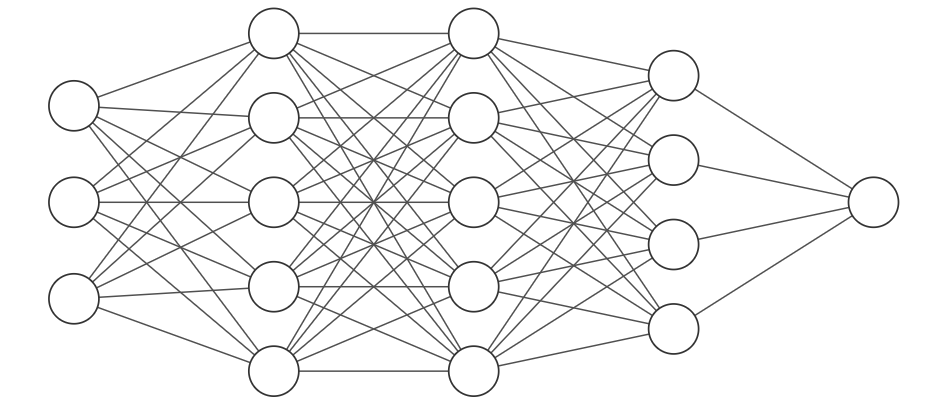}
\caption{\emph{Fully-Connected Network.}
An illustration of a fully-connected network with $N=3$ hidden layers where $d_{in}=3$, $d_1=d_2=5$, $d_3=4$ and $d_{out}=1$.}
\label{fig:fcc}
\end{figure}

Such a network architecture is illustrated in Figure~\ref{fig:fcc}.
From Definition~\ref{def:network} it is easy to see that a fully-connected ReLU network is a continuous piecewise linear function on some polygonal partition of the input domain.
Changing the internal parameters $A^{(k)}, b^{(k)}$ will affect not only the function values computed by the network but also the polygonal partition it is subordinate to in a nontrivial manner.
For simplicity, we will let $d_{out}=1$ so the network computes a real-valued function, which is a convenient choice when solving binary classification problems, and we restrict ourselves to networks of constant width equal to the input dimension, that is, $d_{in}=d_1=d_2=\hdots=d_N=d$. Our approach also generalizes to networks where $d_{in}\geq d_1\geq d_2\geq\hdots\geq d_N$, see Remark~\ref{rem:contracting}.
\begin{lem}[Canonical Network Structure] Any network on the form \eqref{eq:deepnet} where the parameters in each layer are such that the dual basis is well defined (by Definition~\ref{def:dual-basis}, or Remark~\ref{rem:contracting} in the case of a contracting layer) is equivalent to the network
\begin{align}
F(x) = \tilde{L} \circ \tilde{\pi}^{(N)} \circ \cdots \circ \tilde{\pi}^{(1)}(x)
\end{align}
where $\tilde{\pi}^{(k)}$ is the projection onto a polyhedral cone defined by the parameters 
 $\tilde{A}^{(k)}\in \IR^{d_k\times d_{1}}$, $\tilde{b}^{(k)}\in \IR^{d_k}$ of the affine function
\begin{align}
\tilde{A_b}^{(k)}(x)
=
\tilde{A}^{(k)}x+\tilde{b}^{(k)}
=
A_b^{(k)} \circ \cdots \circ A_b^{(1)}(x)
\end{align}
and where $\tilde{L}(x)=L\circ \tilde{A_b}^{(N)}(x)$.
\end{lem}
\begin{proof}
This follows directly from the commutating property \eqref{eq:commutating} of a ReLU layer and the fact that compositions of affine functions are affine functions.
\end{proof}

\subsection{Classification Model Problem}

Let $X_1, X_2 \subset \IR^{d}$ be two sets, each identified with one of two classes, such that $X_1 \cap X_2 = \emptyset$. In general, these two underlying sets are unknown and we are only given a set of samples from them. The classification problem can be formulated as finding a real-valued network $F$ separating the two sets in the following sense
\begin{align}
\begin{cases}
F(x) > 0, \quad \text{if } x \in X_1\\
F(x) < 0, \quad \text{if } x \in X_2
\end{cases}
\end{align} 
Points are then classified based on the sign after evaluating them using $F$. The conditions above split the input domain $\IR^d$ into the super- and sublevel sets, $\{x \in \IR^{d} : F(x) > 0 \}$ and $\{x \in \IR^{d} : F(x) < 0 \}$,
called the decision regions of $F$. These regions define the network classifier in the sense that points will be classified by the network based on which region in $\IR^d$ they belong to. The decision regions are separated by the hypersurface
\begin{align} \label{eq:decision-bdry-general}
\Gamma = \{x \in \IR^{d} \, : \, F(x)=0\}
\end{align}
called the decision boundary of the network $F$. Mathematically, $\Gamma$ is precisely the preimage of $\{0\}$ under $F:\IR^d\rightarrow \IR$. Consequently, a binary classification problem is solved if and only if $F$ changes sign over $\Gamma$, separating the sets $X_1$ and $X_2$.

\paragraph{Condition on ReLU Layers.}
Our geometric interpretation of how fully-connected ReLU layers affect the data in this classification setting gives a fundamental condition on each such layer. The polyhedral cone defined through the parameters of the first layer will intersect $\IR^d$ where the two disjoint sets $X_1$ and $X_2$ live. The sets will then be projected onto parts of the cone, as described above, and mapped affinely to $\IR^d_+$. In this way, the data will evolve through the network by repeatedly applying such transformations for each ReLU layer present in the network. During the actual training of the network, the parameters of each layer are optimized which geometrically means that the shapes of the associated polyhedral cones are changing. Now, let $\tilde{X}_1$ and $\tilde{X}_2$ denote the images of $X_1$ and $X_2$, respectively, under the composition of the first $n$ ReLU layers in a network. Then, if we apply one more ReLU layer $T$ we must have
\begin{equation}
\pi(\tilde{X}_1)\cap\pi(\tilde{X}_2)=\emptyset
\label{eq:not_mixed}
\end{equation}
where $\pi$ is the projection defined by $T$. Otherwise, the two data sets identified with different classes will be mixed; thus, the binary classification problem cannot be solved. The condition in \eqref{eq:not_mixed} restricts the possible positions and orientations of the cones identified with each ReLU layer. Thus, if a network solves the binary classification problem, then \eqref{eq:not_mixed} is a necessary condition for each layer in that network.

\subsection{Decision Boundaries for Shallow Networks}

We continue by analyzing decision boundaries for networks with one hidden layer. Consider a network $F:\IR^d\rightarrow \IR$ on the form
\begin{align} \label{eq:shallow}
F(x)= L \circ T(x)
\end{align}
where the hidden layer $T$ is a fully-connected ReLU layer \eqref{eq:map} and the output layer $L:\IR^d \rightarrow \IR$ is an affine transformation
\begin{align}
L(y) = \hat{a}^L \cdot y + \hat{b}^L
\end{align}
with parameters $\hat{a}^L \in \IR^d$ and $\hat{b}^L \in \IR$.
Note that the non-linear behavior of this network, defined by the composition
$\IR^d \stackrel{T} \longrightarrow \IR^d_+ \stackrel{L} \longrightarrow \IR$,
is contained in the hidden layer $T$.
By the analysis above $T$ reduces to \eqref{eq:TonSector-last} on each sector in $\mcS$, and likewise, its preimage $T^{-1}$ on each sector in $\widehat{\mcS}$ reduces to \eqref{eq:preimsubset}.
We are now interested in detailing the decision boundary \eqref{eq:decision-bdry-general} when $F(x)$ is given by a shallow network \eqref{eq:shallow}.
The output of $T$ will lie on $\overline{\widehat{S}}_{(I,\emptyset)}$ while the input to $L$ generating a zero output will lie on the hyperplane
\begin{align}
\widehat{P} = \mathrm{ker}(L) = \{y\in\IR^d: \hat{a}^L\cdot y + \hat{b}^L=0\}
\end{align}
Hence, we realize that the decision boundary can be expressed as the preimage $T^{-1}$
of the intersection between $\widehat{P}$ and $\overline{\widehat{S}}_{(I,\emptyset)}$.
By further decomposing $\overline{\widehat{S}}_{(I,\emptyset)}$ into sectors in $\widehat{\mcS}$, we arrive at the following expression for the decision boundary to a shallow network.
\begin{align} \label{eq:Gamma-preim}
\Gamma
&
=
\{ x \in \IR^d \,:\, L \circ T (x) = 0 \}
\\&
=
T^{-1} \big(\widehat{P} \cap \overline{\widehat{S}}_{(I,\emptyset)} \big)
\\&
=
T^{-1}\bigg(\widehat{P} \cap \bigcup _{J\subseteq I}\widehat{S}_{(J,\emptyset)}\bigg)
\\&
=
\bigcup_{J\subseteq I} T^{-1}\big(\widehat{P}\cap\widehat{S}_{(J,\emptyset)}\big)
\end{align}
Since we are only interested in non-degenerate cases where $\Gamma$ is an actual $(d-1)$-dimensional surface, i.e., not empty nor filling $d$-dimensional sectors in $\mcS$, we assume the intersection $\widehat{P} \cap \widehat{S}_{(I,\emptyset)}$ to be non-empty.
This implies $0\notin \widehat{P} \Leftrightarrow \hat{b}^L \neq 0$ and to simplify our description below, we also assume that $\hat{b}^L < 0$, which we can do without loss of generality since if this is not the case we can factor out $-1$ from $L$ that we instead incorporate into $T$. 

The hypersurface $\Gamma$ is a continuous piecewise linear surface where each index subset $J\subseteq I$ corresponds to one linear piece of $\Gamma$ that by \eqref{eq:preimsubset} can be expressed
\begin{align}
T^{-1}\big(\widehat{P}\cap \widehat{S}_{(J,\emptyset)}\big)
=
\bigg\{x\in \IR^d \,:\, x=x'-\sum_{i\in I\setminus J}\lambda_ia^*_i,
 \, \lambda_i\geq 0,
 \, x'\in P\cap S_{(J,\emptyset)}\bigg\}
\label{eq:preimage_intersect}
\end{align}
where $P$ is the domain hyperplane given by the preimage of $\widehat{P}$ under the affine transformation~\eqref{eq:A_b}. More explicitly, this hyperplane can be expressed
\begin{align} \label{eq:hyperplaneP}
P=\{x\in\IR^d: A_b(x) \in \widehat{P}\} = \{ x\in \IR^d: a^L \cdot x + b^L = 0 \}
\end{align}
where $a^L \in \IR^d$ and $b^L \in \IR$ are defined
\begin{align}
\begin{split}
a^L &=A^T\hat{a}^L \quad \text{and} \quad
b^L =\hat{a}^L \cdot b + \hat{b}^L
\label{eq:relation}
\end{split}
\end{align}
Hence, $\Gamma$ is completely described by the intersections $P\cap S_{(J,\emptyset)}$ for $J\subseteq I$ and the dual basis $\{a_i^*:i\in I\}$. Of particular interest are the preimages of the intersections with the $(d-1)$-dimensional faces
\begin{align}
T^{-1}\big(\widehat{P} \cap \widehat{S}_{(I\setminus \{i\}, \emptyset)}\big) = \big\{ x\in \IR^d
\,:\,
x=x' - \lambda_i a_i^*
, \
\lambda_i \geq 0
, \
x' \in P\cap S_{(I\setminus \{i\},\emptyset)} \big\}
\label{eq:preimage-face}
\end{align}
where $i\in I$. Since the remaining pieces are linear transitions between these parts $\Gamma$ is completely determined by the preimages in \eqref{eq:preimage-face}. Figure~\ref{fig:hyp_dec} shows an example of how $\Gamma$ is generated given a hyperplane as the kernel of an affine map $L$.

\begin{figure} 
\centering
\begin{subfigure}[b]{0.4\textwidth}
\centering
\includegraphics[width=0.8\textwidth]{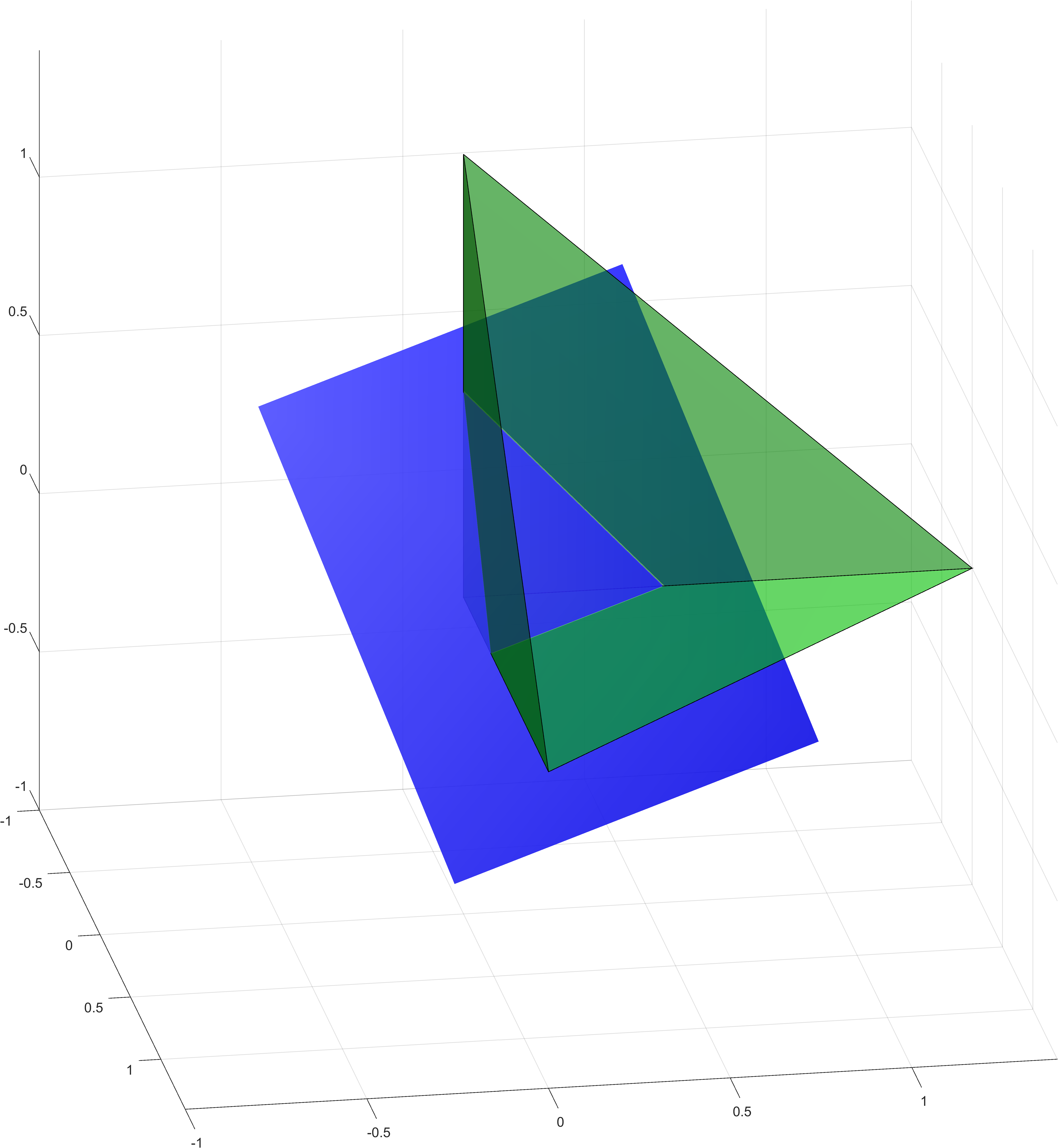}
\caption{ }
\end{subfigure}
\quad
\begin{subfigure}[b]{0.4\textwidth}
\centering
\includegraphics[width=0.8\textwidth]{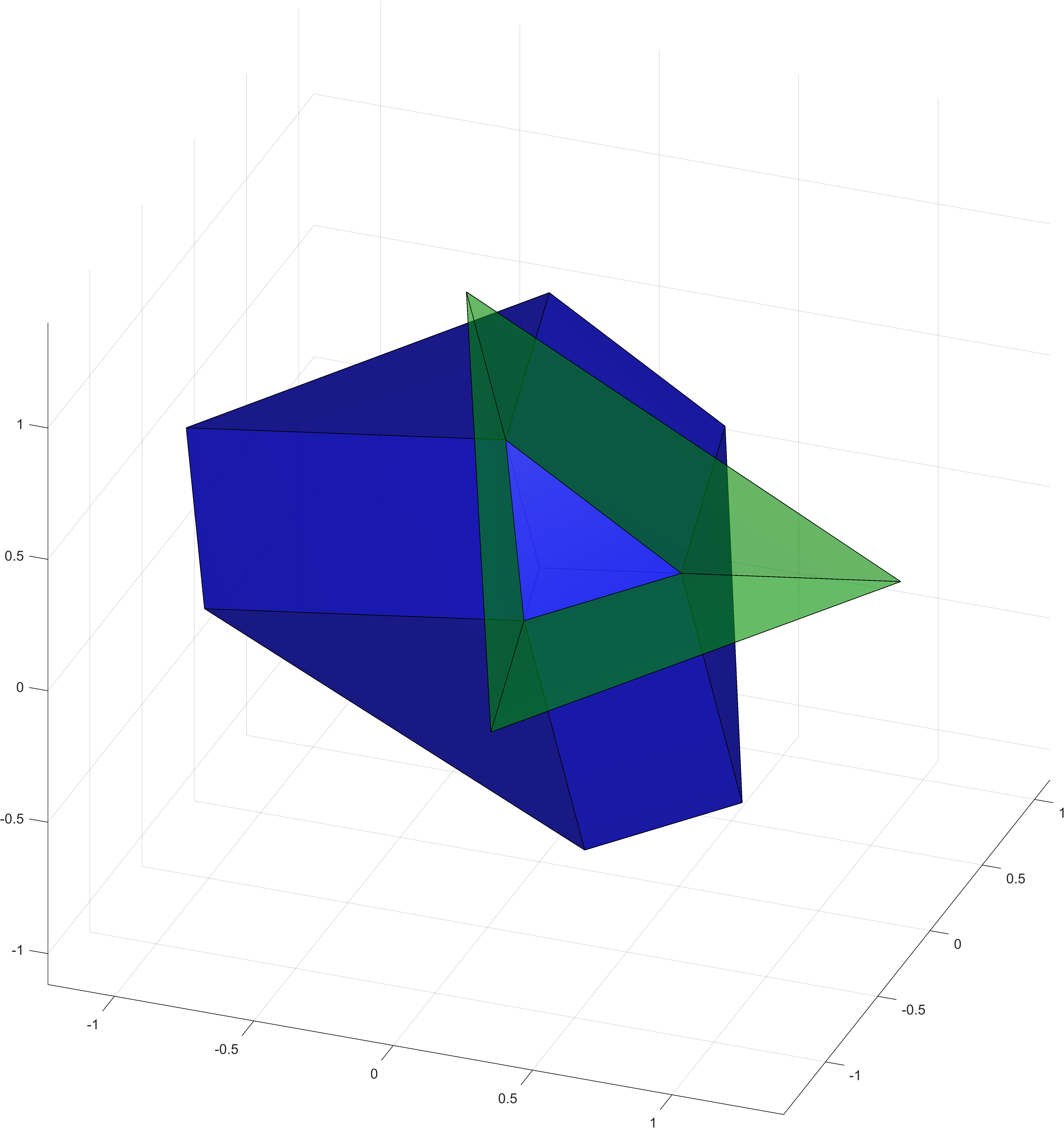}
\caption{ }
\end{subfigure}
\hfill
\caption{\emph{Decision Boundary Construction.}
The creation of the decision boundary $\Gamma$ as the preimage of the hyperplane $\widehat{P}=\text{ker}(L)$ under the ReLU layer $T$. \textbf{(a)} A hyperplane $\widehat{P}$ intersecting the non-negative orthant $\overline{\widehat{S}}_{(I,\emptyset)}=\IR^d_+$. The set $\widehat{P} \cap(\IR^d \setminus \IR^d_+)$ will have an empty preimage. \textbf{(b)} The central piece $\widehat{P}\cap \widehat{S}_{(I,\emptyset)}$ is transformed by the inverse of the affine map to $P\cap S_{(I,\emptyset)}$. Each non-empty intersection of $P$ with the sets in $\partial \mcS$ will generate a linear piece spanned by a subset of the dual vectors. The union of all these linear pieces defines $\Gamma$.}
\label{fig:hyp_dec}
\end{figure}

According to \eqref{eq:preimage-face}, $\alpha_i^*$ is a tangent vector to $T^{-1}(\widehat{P} \cap \widehat{S}_{(I\setminus \{i\})})$
whose direction relative to the central linear piece $P\cap S_{(I,\emptyset)}$ (with normal direction $a^L$) gives an indication on how $\Gamma$ curves. Hence, the signs of the scalar products 
\begin{align}
\{ a_i^* \cdot a^L : i \in I \}
\end{align}
characterize the geometry of the surface. If they are all of one sign, the surface is a boundary of a convex set. Otherwise, it's a saddle surface in the sense that some pieces are curved towards the central piece and others away from it.

\paragraph{Intersection Values.}
Since $0\notin \widehat{P}$ by assumption it follows that $x_0\notin P$ and hence the hyperplane $P$ is completely determined by its intersections with the lines $L_i = \{ x_0 + t a_i^*: t \in \IR \}$. Assuming general position of $P$ (not parallel with any of the lines $L_i$) there, for each $i\in I$, is a parameter $t_i\neq 0$ specifying where the line $L_i$ intersects the hyperplane $P$ such that 
\begin{align} \label{eq:Pintersections}
x_0+t_i a_i^* \in P
\end{align}
Here we require at least one $t_i>0$, because if all $t_i<0$ then $P$ does not intersect $\overline{S}_{(I,\emptyset)}$ at all and hence $\Gamma = \emptyset$, which breaks our assumption that $\Gamma$ is a $(d-1)$-dimensional surface.
Inserting the intersection points \eqref{eq:Pintersections} in the hyperplane equation \eqref{eq:hyperplaneP} gives
\begin{align}
0 = a^L \cdot (x_0 + t_i a_i^*) + b^L = t_i (a^L \cdot a_i^*) + a^L \cdot x_0 + b^L, \quad i\in I
\label{eq:signs}
\end{align}
By the assumption that $\hat{b}^L<0$ and the relations \eqref{eq:relation} it follows that $a^L \cdot x_0 + b^L<0$, which in combination with \eqref{eq:signs} gives the inequality
\begin{align}
t_i (a^L \cdot a_i^*) > 0
\end{align}
and hence, we conclude that the sign of $t_i$ is the same as the sign of $a^L \cdot a_i^*$. We also see that 
\begin{align}
a^L \cdot a_i^*=A^{T}\hat{a}^L \cdot a_i^* = \hat{a}^L \cdot A a_i^*=\hat{a}^L\cdot e_i=\hat{a}^L_i
\end{align}
where $\hat{a}^L_i$ is the $i$:th component of the vector $\hat{a}^L$. Note that $\hat{a}^L$ is an actual training parameter in $L$, whose components' signs directly determine how $\Gamma$ curves relative to the central piece and the signs of the values $t_i$. The values of $t_i$ can also be calculated from the parameters in $L$ through the equation
\begin{align}
t_i=\frac{-\hat{b}^L}{\hat{a}_i^L}
\end{align}

\paragraph{Intersections.} 
In terms of these values we will now describe the intersections $P\cap S_{(J,\emptyset)}$ used in \eqref{eq:preimage_intersect}, the expression for the linear pieces of $\Gamma$.
Let $J\subseteq I$ and recall from \eqref{eq:S} that $x \in S_{(J,\emptyset)}$ has the expansion $x=x_0 + \sum_{j \in J}\alpha_j a_j^*$ with coefficients $\alpha_j >0$.
The intersection $P \cap S_{(J,\emptyset)}$ is the set of points $x \in S_{(J,\emptyset)}$ satisfying
\begin{align}
0
&=
a^L \cdot x + b^L
\\&=
a^L \cdot \bigg(x_0 + \sum_{i\in J}\alpha_i a_i^*\bigg)+b^L
\\&=
a^L \cdot x_0  + \sum_{j\in J} \alpha_j a^L \cdot a_j^*+ b^L
\\&=
(a^L \cdot x_0 + b^L)\bigg(1-\sum_{j\in J} \frac{\alpha_j}{t_j} \bigg)
\label{eq:twoparenth}
\end{align}
where we in the last equality used the identity $a^L \cdot a^*_j=-\frac{a^L \cdot x_0 + b^L}{t_j}$ deduced from \eqref{eq:signs}.
Since $x_0\notin P \Leftrightarrow (a^L \cdot x_0+ b^L)\neq 0$, the second parenthesis in \eqref{eq:twoparenth} must be zero, yielding the additional condition
\begin{align}
\sum_{j\in J} \frac{\alpha_j}{t_j} =1
\label{eq:condition}
\end{align}
on the coefficients $\alpha_j>0$. 
This means that the intersection $P\cap S_{(J,\emptyset)}$ can be expressed
\begin{align}
\label{eq:intersection}
P\cap S_{(J,\emptyset)}=\bigg\{x\in \IR^d \,:\, x=x_0+\sum_{j\in J}\alpha_j a_j^*,
\, \alpha_j>0,
\, \sum_{j\in J}\frac{\alpha_j}{t_j}=1\bigg\}
\end{align}
Introducing the index set $J^-=\{j\in J: t_j < 0\}$, we note that in case $J^-=J$ the condition \eqref{eq:condition} cannot be fulfilled, and hence the intersection $P\cap S_{(J,\emptyset)}=\emptyset$. At the other extreme, when $J^-=\emptyset$, all terms in \eqref{eq:condition} are strictly positive and we can deduce that $\alpha_j<t_j$, which implies that the intersection is non-empty and bounded. In the remaining case where $\emptyset \subset J^- \subset J$, we have both positive and negative terms in \eqref{eq:condition}, which implies that the intersection is non-empty and unbounded. Due to the additional condition \eqref{eq:condition}, all non-empty intersections will be sets of dimension $|J|-1$.

Combining \eqref{eq:intersection} with \eqref{eq:preimage_intersect} gives the following expression for the preimages we are interested in. For all $J\subseteq I$ such that $J^- \neq J$ we get non-empty preimages
\begin{align}
\label{eq:preimages_intersections}
&T^{-1}\big(\widehat{P}\cap \widehat{S}_{(J,\emptyset)}\big)=\\
&\ \  \bigg\{x\in \IR^d \,:\, x=x_0+\sum_{j\in J}\alpha_ja^*_j-\sum_{i\in I\setminus J}\lambda_ia^*_i,
\, \lambda_i\geq 0,
\, \alpha_j>0,
\, \sum_{j\in J} \frac{\alpha_j}{t_j}=1\bigg\}
\nonumber
\end{align}
and if $J^- = J$ the preimage is empty.

Using these descriptions we can now calculate the number of linear pieces of $\Gamma$.
\begin{thm}[Number of Linear Pieces]
\label{thm:linearpieces}
The number of linear pieces of the decision boundary $\Gamma=\{x\in \IR^d: L \circ T(x)=0\}$ of a fully-connected ReLU network $F:\IR^d\rightarrow \IR$ with one layer is
\begin{align}
\boxed{
2^d-2^m
}
\end{align}
where $m=|\{i\in I: t_i<0\}|$.
\end{thm}
\begin{proof}
By \eqref{eq:preimage_intersect} we see that each non-empty intersection $P\cap S_{(J,\emptyset)}$ will generate a unique linear piece of $\Gamma$. Hence, the number of non-empty intersections will determine the number of linear pieces of $\Gamma$. We have showed that the only empty intersections are those $P\cap S_{(J,\emptyset)}$ where $J\subseteq \{i\in I: t_i<0\}$. The number of such subsets is exactly $2^{m}$ where $m=|\{i\in I: t_i<0\}|$ and since $|I|=d$ we conclude that the total number of linear pieces is exactly $2^d-2^{m}$. 
\end{proof}

Theorem~\ref{thm:linearpieces} holds as long as $A$ has full rank and the hyperplane $P$ is in general position as defined before. Thus, the number of linear pieces is directly determined by $m$, i.e., the number of negative $t_i$, $i\in I$. Since $t_i=\frac{-\hat{b}^L}{\hat{a}_i^L}$ the number of linear pieces can be directly computed from the training parameters in the affine function $L$. However, we saw earlier that the signs of the values $t_i$ also control how $\Gamma$ curves relative to the central piece $P\cap S_{(I,\emptyset)}$. Thus, maximizing the number of linear pieces will restrict the potential complexity of the geometry of $\Gamma$. For example, when $m=0$ (i.e, all $t_i$ are positive), we will maximize the number of linear pieces of $\Gamma$ to $2^d-1$ but then all the dot products $\{a_i^*\cdot n\}_{i\in I}$ are positive (since the sign of $a_i^*\cdot n$ is the same as the sign of $t_i$ according to \eqref{eq:signs}), so $\Gamma$ will be a convex hypersurface. This shows that there is a trade-off between the complexity in terms of the number of linear pieces and the complexity in terms of the curvature of the decision boundary.

\begin{definition}[Canonical Decision Boundaries] \label{def:canonical-decision-bdry}
Let $\widehat{F}$ be a shallow network \eqref{eq:shallow} where the parameters in the ReLU layer \eqref{eq:map}, which we denote $\widehat{T}$, are $A=I_d$ and $b=0$. In this special case $a_i^* = e_i$, $x_0=0$, $\mcS = \widehat{\mcS}$, and any hyperplane $P=\widehat{P}$. By \eqref{eq:Pintersections}, there for each hyperplane in general position are values $\{t'_i\in\IR : i \in I\}$ such that the points $t'_i e_i \in \IR^d$ define the intersection between the hyperplane and the affine sector.
For each $m \in \{ 0,1,\dots,d-1 \}$, we let $P_m=\widehat{P}_m$ be the unique hyperplane yielding intersection values
\begin{align} \label{eq:canonical-t}
t_i' &=\begin{cases} -1 \quad &\text{if $i\leq m$}\\
1 \quad &\text{if $i>m$}
\end{cases}
\end{align}
Note that we include a prime in the notation for variables defining canonical decision boundaries to simplify later comparisons to arbitrary decision boundaries.
Assuming the final affine transformation $L$ in our network $\widehat{F}$ is such that $\mathrm{ker}(L) = \widehat{P}_m$, the decision boundary induced by $\widehat{F}(x)=0$ is
\begin{align}
\widehat{\Gamma}_m= \bigcup_{J\subseteq I} \widehat{T}^{-1}\big(\widehat{P}_m\cap \widehat{S}_{(J,\emptyset)}\big)
\label{eq:Gamma_m}
\end{align}
where we by \eqref{eq:preimages_intersections} can express the preimage of each intersection 
\begin{align}
\label{eq:preimage-canonical}
&\widehat{T}^{-1}\big(\widehat{P}_m\cap \widehat{S}_{(J,\emptyset)}\big)=\\ 
&\quad \quad\quad\bigg\{x\in \IR^d \,:\, x=\sum_{j\in J}\alpha_j' e_j-\sum_{i\in I\setminus J}\lambda_i' e_i,\, \lambda_i' \geq 0,\, \alpha_j' >0,\, \sum_{j\in J} \frac{\alpha_j'}{t_j'}=1\bigg\} \nonumber
\end{align}
We call $\bigl\{\widehat{\Gamma}_m : m = 0,\dots,d-1 \bigr\}$ the set of \emph{canonical decision boundaries}.
\end{definition}

By definition, there are $d$ canonical decision boundaries, which for the case $d=3$ are illustrated in Figure~\ref{fig:canonical_surf}. As can be seen from \eqref{eq:canonical-t} the integer $m$ conforms with the prior definition $m=|\{i\in I: t_i<0\}|$, and thus we can conclude that $\widehat{\Gamma}_m$ consists of $2^d-2^m$ linear pieces according to Theorem~\ref{thm:linearpieces}. We will next show that every decision boundary induced by a shallow network \eqref{eq:shallow} is equivalent to one canonical decision boundary in the following sense.

\begin{figure} 
\centering
\begin{subfigure}[b]{0.32\textwidth}
\centering
\includegraphics[trim=135px 185px 115px 165px, clip ,width=\textwidth]{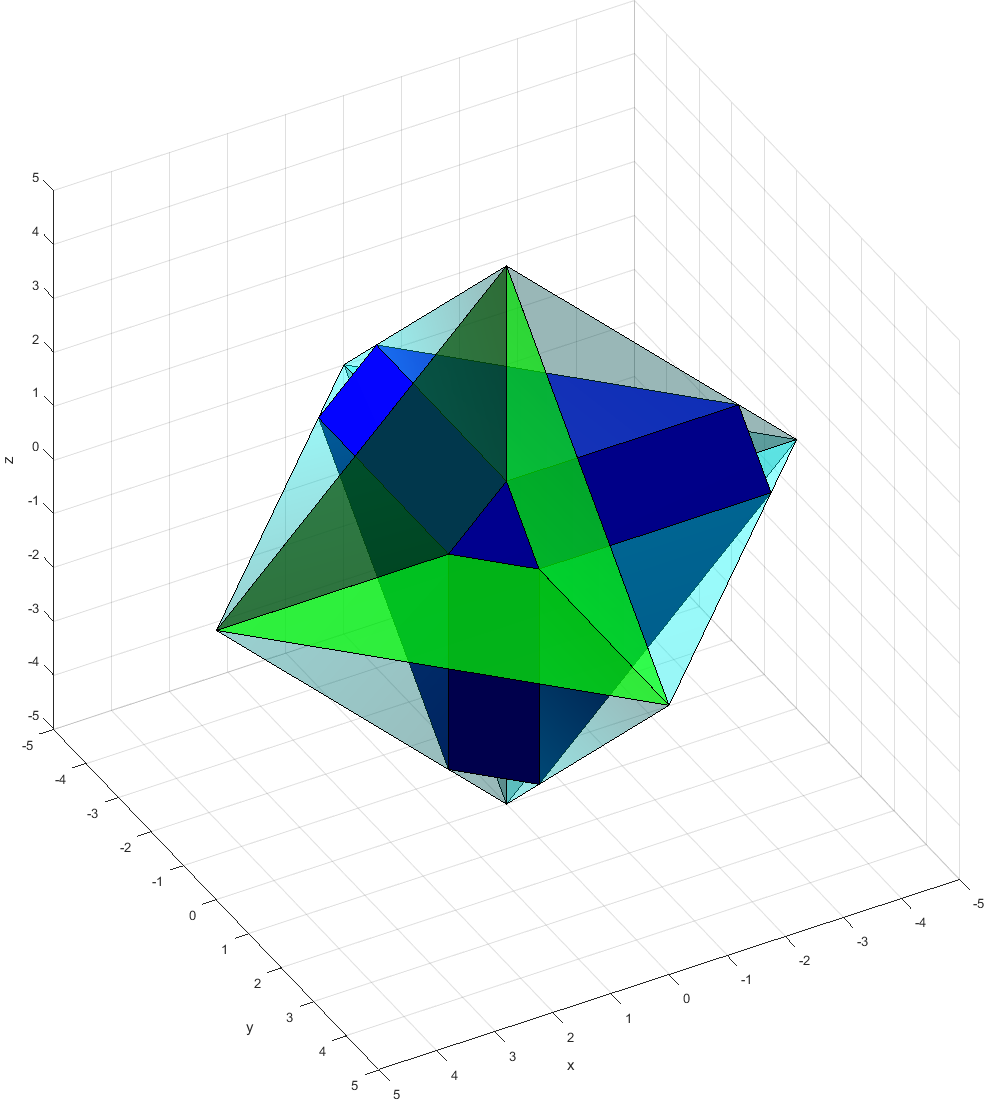}
\caption{ }
\end{subfigure}
\
\begin{subfigure}[b]{0.32\textwidth}
\centering
\includegraphics[trim=135px 185px 115px 165px, clip ,width=\textwidth]{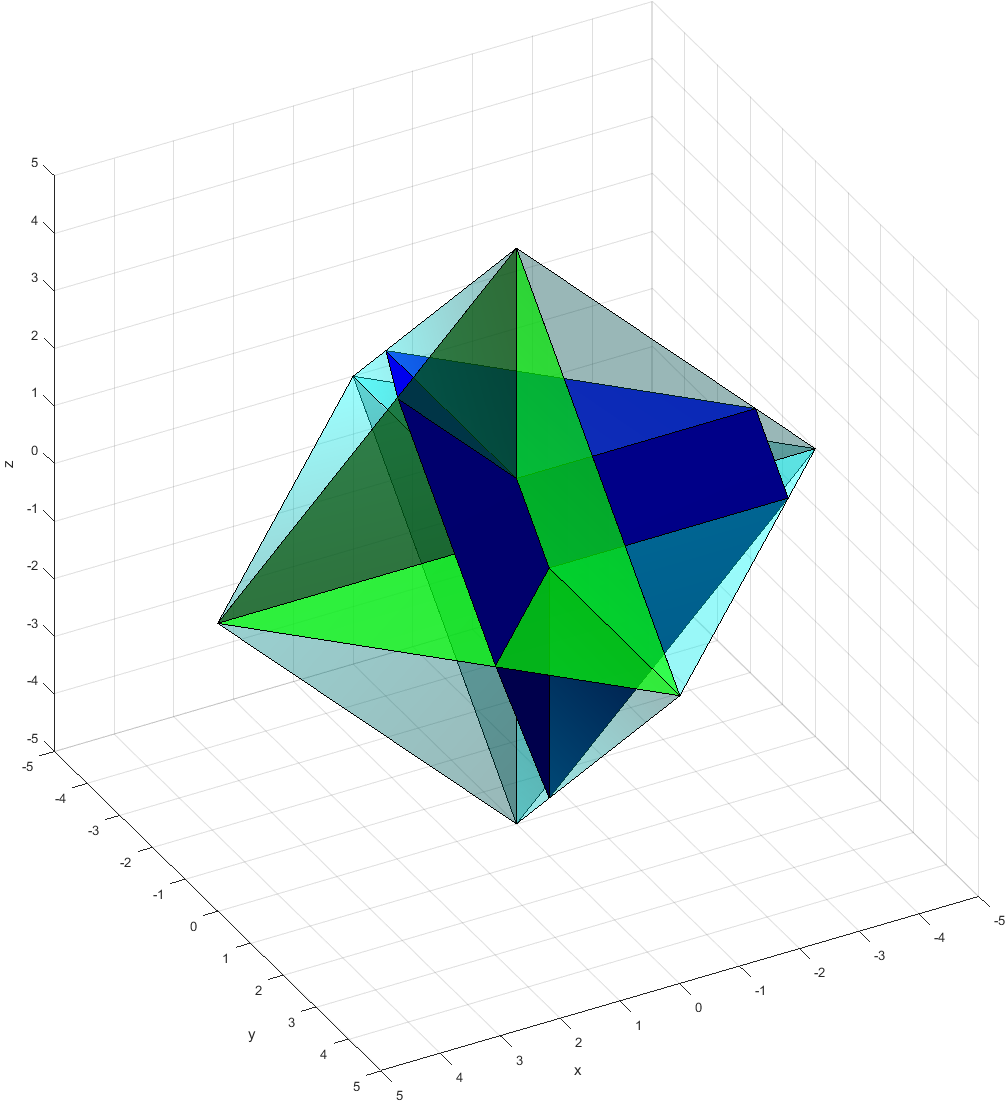}
\caption{ }
\end{subfigure}
\
\begin{subfigure}[b]{0.32\textwidth}
\centering 
\includegraphics[trim=135px 185px 115px 165px, clip ,width=\textwidth]{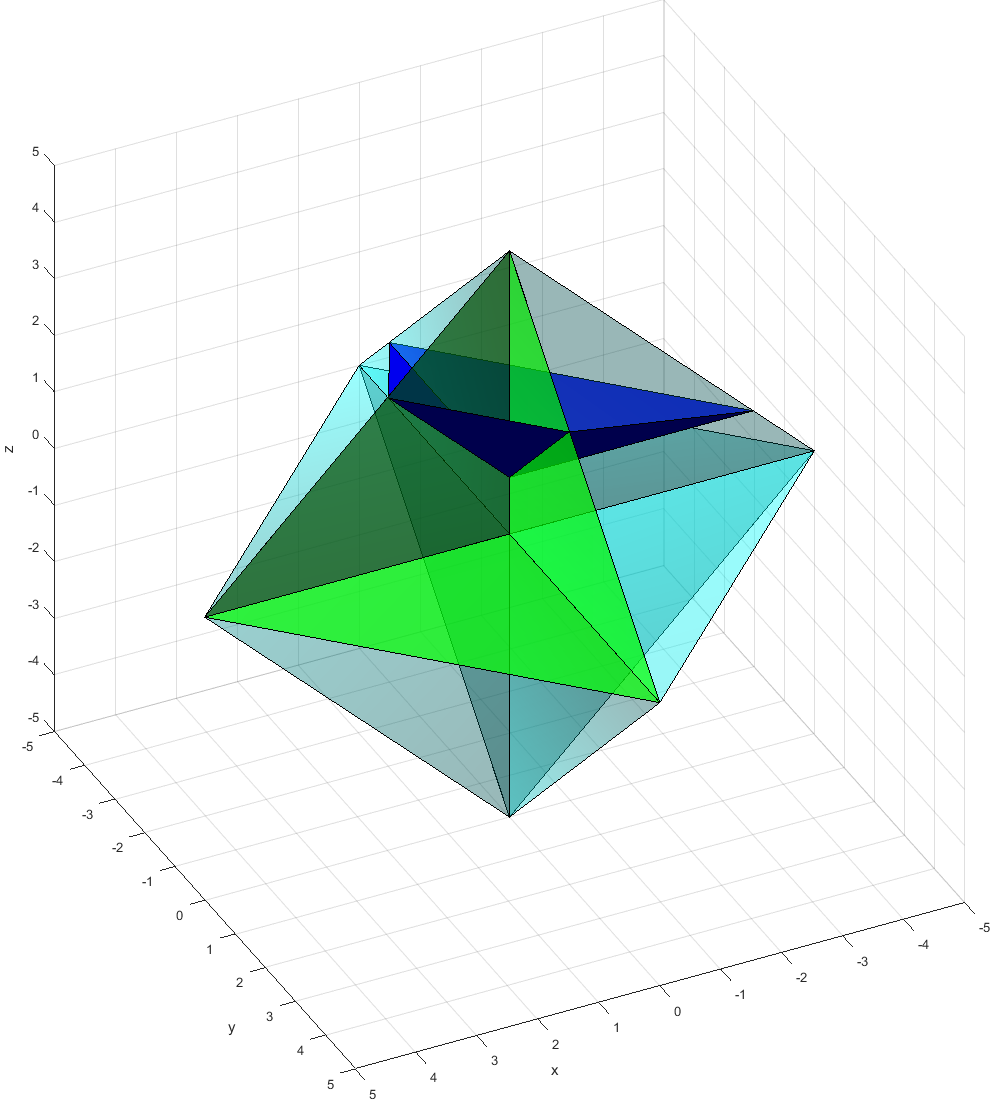}
\caption{ }
\end{subfigure}
\caption{\emph{Canonical Decision Boundaries.} For a network $F:\IR^3\rightarrow \IR$ with one fully-connected ReLU layer there are $d=3$ canonical decision boundaries. \textbf{(a)} An illustration of $\widehat{\Gamma}_0$. Since all $t'_i$ are positive, the hyperplane $\widehat{P}_0$ intersects all $\widehat{S}_{\bfI}\in \partial \widehat{S}$ except for $\widehat{S}_{(\emptyset,\emptyset)}=\{0\}$ and therefore generates the maximum of $2^3-2^0=7$ linear pieces. However, since all $t'_i$ are positive, all pieces are curved away from the central piece $\widehat{P}_0\cap \widehat{S}_{(I,\emptyset)}$ and create a convex surface. \textbf{(b)} An illustration of $\widehat{\Gamma}_1$ with $2^3-2^1=6$ linear pieces. Here, the sign of $t'_1$ is negative whereas both $t'_2$ and $t'_3$ are positive, so some of the pieces curve towards the central piece and others away from it. \textbf{(c)} An illustration of $\widehat{\Gamma}_2$ with $2^3-2^2=4$ linear pieces all curving towards the central piece.}
\label{fig:canonical_surf}
\end{figure}

\begin{definition}[Equivalence of Decision Boundaries] 
\label{def:equivalence}
Two decision boundaries $\Gamma$, $\Gamma'\subset \IR^d$ are \emph{equivalent} if there exists an invertible affine map $M:\IR^d\rightarrow \IR^d$ such that $M(\Gamma')=\Gamma$.
\end{definition}

\begin{thm}
\label{thm:equivalence}
A fully-connected ReLU network $F:\IR^d\rightarrow \IR$ with one hidden layer can only generate $d$ non-equivalent decision boundaries.
\end{thm}
\begin{proof}
As above we consider the generic setting when the network $F$ generates a decision boundary $\Gamma$, whose intersection values $\{t_i\in\IR : i \in I\}$ in \eqref{eq:Pintersections} are such that $t_i \neq 0$ and $m=|\{i\in I: t_i<0\}| < d$.
We will now show that this $\Gamma$ is equivalent to the canonical decision boundary $\widehat{\Gamma}_m$ of Definition~\ref{def:canonical-decision-bdry} by constructing an invertible affine transformation $M$ that maps $\widehat{\Gamma}_m$ onto $\Gamma$.
Consider the map $M:\IR^d\rightarrow \IR^d$ given by 
\begin{align}
M(x)=x_0 + A^{-1}D Q_\sigma x
\label{eq:Map-M}
\end{align}
where $x_0=-A^{-1}b$ is the projection cone apex, $D$ is a diagonal matrix with elements $D_{ii}=|t_i|$, and $Q_\sigma$ is a permutation matrix.
The action of the linear part of this map on a basis vector $e_i$ by \eqref{eq:dual} is
\begin{align}
A^{-1}DQ_\sigma e_i= A^{-1}|t_{\sigma(i)}|e_{\sigma(i)}=|t_{\sigma(i)}|a_{\sigma(i)}^*
\label{eq:actonei}
\end{align}
The permutation matrix $Q_\sigma$ is constructed such that it realizes a permutation $Q_\sigma e_i=e_{\sigma(i)}$ where the bijective map $\sigma: I\mapsto I$ has the effect of sorting the intersection values of $\Gamma$, i.e., that $t_{\sigma(i)}\leq t_{\sigma(i+1)}$ for $i=1,\dots,d-1$.
By this construction the signs of the permuted intersection values correspond to those of the canonical decision boundary such that  $\mathrm{sign}(t_{\sigma(j)})=\mathrm{sign}(t_j')$ and since $t'_j\in \{-1,1\}$ we have the relation
\begin{align} \label{eq:intersectrel}
t'_j|t_{\sigma(j)}| = t_{\sigma(j)}
\end{align}
We will now show that the action of \eqref{eq:Map-M} is such that
\begin{align} \label{eq:MGammaEquiv}
M(\widehat{\Gamma}_m)
&=
\bigcup_{J\subseteq I} M\bigl( \widehat{T}^{-1}\big(\widehat{P}_m\cap \widehat{S}_{(J,\emptyset)}\big) \bigr)
=
\bigcup_{K\subseteq I} \widehat{T}^{-1}\big(\widehat{P}\cap \widehat{S}_{(K,\emptyset)}\big)
=
\Gamma
\end{align}
and, hence, that $\Gamma$ and $\widehat{\Gamma}_m$ are equivalent according to Definition~\ref{def:equivalence}.
For some $J\subseteq I$, consider the expression of the intersection preimage $\widehat{T}^{-1}\big(\widehat{P}_m\cap \widehat{S}_{(J,\emptyset)}\big)$ given in \eqref{eq:preimage-canonical}. Letting $M$ act on the coordinate expansion in \eqref{eq:preimage-canonical} and using \eqref{eq:actonei} give
\begin{align} \label{eq:Mexpansion}
M(x)
&= x_0 + A^{-1}DQ_\sigma \biggl( \sum_{j\in J}\alpha'_je_j-\sum_{i\in I\setminus J}\lambda'_i e_i \biggr)
\\&=
x_0 + \sum_{j\in J} \alpha_{\sigma(j)} |t_{\sigma(j)}|a_{\sigma(j)}^*-\sum_{i\in I\setminus J}\lambda_{\sigma(i)} |t_{\sigma(i)}|a_{\sigma(i)}^*
\\&=
x_0 + \sum_{j\in \sigma(J)} \alpha_{j} |t_{j}|a_{j}^*-\sum_{i\in I\setminus \sigma(J)}\lambda_{i} |t_{i}|a_{i}^*
\end{align}
where we defined $\alpha_{\sigma(j)}=\alpha'_j|t_{\sigma(j)}|$, $\lambda_{\sigma(i)}=\lambda'_i|t_{\sigma(i)}|$ and $\sigma(J)=\{\sigma(j):j\in J\}$.
In terms of these definitions, we can also reformulate the constraints in \eqref{eq:preimage-canonical} as
\begin{align}
\lambda_i'\geq 0 \Leftrightarrow \lambda_{\sigma(i)} \geq 0 ,
\quad
\alpha_j'>0 \Leftrightarrow \alpha_{\sigma(j)}>0 ,
\quad
\sum_{j\in J}\frac{\alpha'_j}{t_j'} = 1 \Leftrightarrow
\sum_{j\in \sigma(J)}\frac{\alpha_{j}}{t_{j}} = 1
\end{align}
The first two equivalences hold since $|t_{\sigma(i)}|>0$ and the last equivalence follows from the calculation
\begin{align} \label{eq:Mcond}
\sum_{j\in J}\frac{\alpha'_j}{t'_j}
= \sum_{j\in J}\frac{\alpha_{\sigma(j)}}{t'_j|t_{\sigma(j)}|}
= \sum_{j\in J}\frac{\alpha_{\sigma(j)}}{t_{\sigma(j)}}
= \sum_{j\in \sigma(J)}\frac{\alpha_{j}}{t_{j}}
\end{align}
where we in the second equality used \eqref{eq:intersectrel}.
Comparing \eqref{eq:Mexpansion}--\eqref{eq:Mcond} to the expression for a general intersection preimage \eqref{eq:preimages_intersections} we realize that
\begin{align}
M\bigl( \widehat{T}^{-1}\big(\widehat{P}_m\cap \widehat{S}_{(J,\emptyset)}\big) \bigr)
= \widehat{T}^{-1}\big(\widehat{P} \cap \widehat{S}_{(\sigma(J),\emptyset)}\big)
\end{align}
Since $\sigma$ is a bijection we have $\sigma(J)\neq\sigma(K)$ for $J,K \subseteq I$ with $J\neq K$ and for every $K\subseteq I$ there is a $J\subseteq I$ such that $\sigma(J)=K$.
Hence, we conclude that \eqref{eq:MGammaEquiv} holds.

We end this proof by noting that the $d$ canonical decision boundaries are non-equivalent. Since an invertible affine map acting on a decision boundary cannot change the number of linear pieces, which by Theorem~\ref{thm:linearpieces} is different for each $m$, any pair of canonical decision boundaries $\widehat{\Gamma}_{m_1},\widehat{\Gamma}_{m_2}$ with $m_1 \neq m_2$ cannot be equivalent.
This gives us that for a fully-connected ReLU network with one hidden layer, there are only $d$ different non-equivalent decision boundaries $\Gamma$ (one for each value of $m$).
\end{proof}

This result implies that if $\Gamma \subseteq \IR^d$ is a decision boundary separating two classes of points $X_1$ and $X_2$, then there is an affine map $M:\IR^d\rightarrow \IR^d$ such that the transformed sets $M(X_1)$ and $M(X_2)$ are separated by $\Gamma_m$ for some $0\leq m<d$. In this sense, it is enough to analyze the properties of the canonical decision boundaries since every other possible decision boundary can be obtained through an affine transformation.

\subsection{Decision Boundaries for Deep Networks}
When appending additional ReLU layers to the network any precise complexity estimates for decision boundary, such as those provided for shallow networks in Theorem~\ref{thm:linearpieces} and Theorem~\ref{thm:equivalence}, are difficult to show.
Instead, we here provide some more general commentary on the effect of network depth.

Consider a network on the form \eqref{eq:deepnet} with $N$ hidden ReLU layers and width $d$. Define
\begin{align}
\Gamma^{(k)}=\big\{x\in \IR^d: L \circ T^{(N)} \circ \hdots \circ T^{(k)}(x)=0 \big\}
,\quad\text{where $1\leq k \leq N$}
\end{align}
We can interpret $\Gamma^{(k)}$ as the decision boundary of the network when we have removed the first $k-1$ layers, so the actual decision boundary is simply $\Gamma=\Gamma^{(1)}$. Removing all but one hidden layer gives a shallow network, which means $\Gamma^{(N)}$ is characterized in the previous section. Moreover, we have the recursive relation 
\begin{align}
\boxed{
\Gamma^{(k)}={T^{-(k)}}\big(\Gamma^{(k+1)}\cap \overline{\widehat{S}}_{(I,\emptyset)}\big)
}
\end{align}
where $T^{-(k)}(\hat{\omega})$ denotes the preimage of $\hat{\omega}$ under the ReLU layer $T^{(k)}$.
Clearly, only the parts of $\Gamma^{(k+1)}$ intersecting $\overline{\widehat{S}}_{(I,\emptyset)}$ will contribute to $\Gamma^{(k)}$. Parts of $\Gamma^{(k+1)}$ entirely inside $\widehat{S}_{(I,\emptyset)}$ will be obtained through an affine transformation so any additional complexity of the hypersurface $\Gamma^{(k)}$ is due to the intersections of $\Gamma^{(k+1)}$ with the boundary $\partial \widehat{S}_{(I,\emptyset)}$.

Unlike a shallow network, where every non-empty intersection of $\mathrm{ker}(L)$ with a boundary part $\widehat{S}_{(J,\emptyset)}\in \partial \widehat{\mcS} $ generates one addidional linear piece, the intersection of $\Gamma^{(k+1)}$ with $\widehat{S}_{(J,\emptyset)}$ can yield more than one additional linear piece in $\Gamma^{(k)}$. The reason is that the intersection $\Gamma^{(k+1)}\cap \widehat{S}_{(J,\emptyset)}$ is typically piecewise linear where each linear piece will generate an additional new linear piece in $\Gamma_k$. However, all those linear pieces emerging from the intersection with the same boundary part $\widehat{S}_{(J,\emptyset)}$ will not be independent since they are spanned by the same set of dual basis vectors. 

In effect, the number of linear pieces of the decision boundary $\Gamma$ can be very large in a network with multiple layers, but they are not entirely independent of each other. Even though one layer can induce more than $d$ linear pieces, all of them are spanned by subsets of only $d$ new vectors, namely the dual basis induced by the parameters of $T^{(k)}$, see Definition~\ref{def:dual-basis}. The overall complexity of $\Gamma$ emerges from these dual bases which will span the additional linear pieces from each layer, together with the increasing complexity of the intersections of $\Gamma^{(k)}$ with $\overline{\widehat{S}}_{(I,\emptyset)}$ as we successively proceed through the layers.

\section{Conclusion} \label{sec:conclusion}
We have provided a detailed geometric description of the structure of a fully-connected ReLU layer by introducing a partition of the input space using a dual basis derived from the layer parameters. With this framework, we can describe the action of such a layer as a projection of $\IR^d$ onto a polyhedral cone followed by an invertible affine transformation mapping the cone onto the non-negative orthant $\IR^d_+$. Most of the sectors in our partition will be projected onto lower-dimensional parts of the boundary of the cone, and hence, it is apparent that networks with multiple ReLU layers can contract data efficiently. However, in a classification setting the cones associated with each layer must be constructed such that the sets of points from different classes are not mixed.

With this geometrical description, we can compute preimages of sets in terms of their intersections with $\IR^d_+$ and the dual vectors. The decision boundary for a network with a single hidden layer can be expressed as the preimage of a hyperplane. This allowed us to characterize the complexity of the decision boundary in terms of the number of linear pieces and the geometry as their relative orientation. We can conclude that there is a trade-off between the number of linear pieces and the dependency between them. In particular, maximizing the number of linear pieces will always result in a convex decision boundary.

We have also classified the possible decision boundaries for a fully-connected ReLU network with a single hidden layer by posing mild conditions on the parameters. All such decision boundaries can be mapped by an affine transformation to one of the $d$ canonical decision boundaries we described in detail. Lastly, we briefly discussed the effect of adding more layers to the network. Specifically, the number of linear pieces of a decision boundary of a network can grow very fast with the number of layers, but the majority of these pieces are highly dependent. 

In an upcoming paper, we provide a geometric approximation theory for deep ReLU networks relevant to the binary classification setting. More precisely, we show that a sufficiently regular hypersurface can be approximated by the decision boundary of a deep ReLU network to any desired accuracy. That construction will heavily depend on the geometric description of ReLU layers provided in this paper.

\bigskip
\paragraph{Acknowledgement.} This research was supported in part by the Wallenberg AI, Autonomous Systems and
Software Program (WASP) funded by the Knut and Alice Wallenberg Foundation; the Swedish Research
Council Grants Nos.\  2017-03911,  2021-04925;  and the Swedish
Research Programme Essence.

\bibliographystyle{habbrv}
\footnotesize{
\bibliography{ref_geom}
}

\bigskip
\noindent
\footnotesize {\bf Authors' addresses:}

\smallskip
\noindent
Jonatan Vallin  \quad \hfill \addressumushort\\
{\tt jonatan.vallin@umu.se}

\smallskip
\noindent
Karl Larsson, \quad \hfill \addressumushort\\
{\tt karl.larsson@umu.se}

\smallskip
\noindent
Mats G. Larson,  \quad \hfill \addressumushort\\
{\tt mats.larson@umu.se}

\end{document}